\documentclass{article} %
\usepackage{iclr2027_conference,times}
\usepackage[acronym,shortcuts]{glossaries} %
\glsdisablehyper %

\usepackage{amsmath,amsfonts,bm}

\newacronym{mfmc}{MFMC}{multi-fidelity Monte Carlo}
\newacronym{em}{EM}{expectation maximization}
\newacronym{ppo}{PPO}{proximal policy optimization}
\newacronym{sac}{SAC}{soft actor-critic}
\newacronym{drl}{DRL}{deep reinforcement learning}
\newacronym{il}{IL}{imitation learning}
\newacronym{rl}{RL}{reinforcement learning}
\newacronym{cv}{CV}{control variate}
\newacronym{mfpg}{MFPG}{multi-fidelity policy gradient}
\newacronym{pg}{PG}{policy gradient}
\newacronym{auc}{AUC}{area under curve}
\newacronym{hf}{HF}{high-fidelity}
\newacronym{lf}{LF}{low-fidelity}
\newacronym{utd}{UTD}{update-to-data}
\newacronym[longplural={Markov decision processes}]{mdp}{MDP}{Markov decision process}
\newacronym{gae}{GAE}{generalized advantage estimation}
\newacronym{ema}{EMA}{exponential moving average}
\newcommand{\probP}{\text{I\kern-0.15em P}}
\newcommand{\policy}{\pi}
\newcommand{\policyparams}{\theta}

\newcommand{\policyParamsDim}{d}

\newcommand{\state}{s}
\newcommand{\action}{a}
\newcommand{\trajectory}{\tau}

\newcommand{\mdp}{\mathcal{M}}
\newcommand{\stateSpace}{S}
\newcommand{\actionSpace}{A}
\newcommand{\initialState}{\state_{I}}
\newcommand{\initialStateDist}{\Delta_{\initialState}}

\newcommand{\given}{\mid}

\newcommand{\RV}{X}

\newcommand{\high}{\mathrm{high}}
\newcommand{\low}{\mathrm{low}}

\newcommand{\samplecost}{w}
\newcommand{\samplecosth}{\samplecost^{\high}}
\newcommand{\samplecostl}{\samplecost^{\low}}

\def\eqref#1{equation~\ref{#1}}

\def\1{\bm{1}}

\DeclareMathAlphabet{\mathsfit}{\encodingdefault}{\sfdefault}{m}{sl}
\SetMathAlphabet{\mathsfit}{bold}{\encodingdefault}{\sfdefault}{bx}{n}

\newcommand{\E}{\mathbb{E}}

\newcommand{\R}{\mathbb{R}}

\newcommand{\Var}{\mathrm{Var}}

\newcommand{\Cov}{\mathrm{Cov}}

\RequirePackage{xcolor}
\newcommand{\hfc}[1]{\textcolor{black}{#1}}
\newcommand{\lfc}[1]{\textcolor{black}{#1}}

\usepackage{hyperref}
\usepackage{url}
\usepackage[most]{tcolorbox}
\usepackage{amsthm}
\theoremstyle{plain}
\newtheorem{lemma}{Lemma}
\theoremstyle{remark}

\usepackage{algorithm}
\usepackage{algpseudocode}
\usepackage{subcaption}
\usepackage{booktabs,array}
\usepackage[capitalize]{cleveref} %

\graphicspath{{./Images/}}%
\title{Multi-Fidelity Policy Gradients Stabilize Data-Scarce Reinforcement Learning}

\author{Xinjie Liu \\
University of Texas at Austin \\
\texttt{xinjie-liu@utexas.edu} \\
\And
\makebox[\dimexpr0.5\textwidth-2\tabcolsep\relax][l]{Ruihan Zhao} \\
University of Texas at Austin \\
\texttt{ruihan.zhao@utexas.edu} \\
\AND
Anirban Chaudhuri \\
University of Texas at Austin \\
\texttt{anirbanc@oden.utexas.edu} \\
\And
\makebox[\dimexpr0.5\textwidth-2\tabcolsep\relax][l]{Cyrus Neary} \\
University of British Columbia \\
\texttt{cyrus.neary@ubc.ca} \\
\AND
Ufuk Topcu \\
University of Texas at Austin \\
\texttt{utopcu@utexas.edu} \\
\end{tabular}\hfill\begin{tabular}[t]{l}\bf\rule{0pt}{24pt}\ignorespaces
\makebox[\dimexpr0.5\textwidth-2\tabcolsep\relax][l]{David Fridovich-Keil} \\
University of Texas at Austin \\
\texttt{dfk@utexas.edu}
}

\iclrfinalcopy %
\begin{document}

\maketitle
\ificlrfinal\lhead{Preprint}\fi
\vspace{-0.25in}
\begin{center}
\textbf{Project page:} \url{https://xinjie-liu.github.io/mfpg-ppo/}
\end{center}
\vspace{0.05in}

\begin{abstract}

Policy gradient methods for on-policy \ac{rl} can become unstable when expensive, scarce target-domain data yield noisy gradient estimates. We address this challenge by complementing limited \emph{\ac{hf}} target-domain data with abundant, cheap, but biased \emph{\ac{lf}} data, e.g., from a simplified simulator.
Most existing methods directly optimize biased objectives based on \ac{lf} data. In contrast,
the recently introduced \ac{mfpg} framework uses \ac{lf} data solely to construct a control variate that reduces variance and improves \ac{hf} data efficiency without biasing the policy gradient estimator. 
However, published work on \ac{mfpg} is limited to REINFORCE on small-scale simulation tasks. 
We develop \ac{mfpg} for modern actor-critic learning in GPU-parallel simulation and on a physical robot.  
Our analysis and experiments show that naive extensions to \ac{ppo} can lose cross-fidelity correlation or inflate variance. 
Our \acs{mfpg}-\acs{ppo} addresses these failures by redesigning the sampling, advantage estimation, and control variate construction to preserve cross-fidelity correlation, and by monitoring estimator uncertainty to prevent variance inflation.
We also introduce a budget-aware \ac{mfpg}-\ac{ppo} to divide a fixed sampling budget among high- and low-fidelity data sources. Across simulated robot locomotion tasks of varying \ac{lf}-to-\ac{hf} transfer difficulty and \ac{hf} data budgets, \ac{mfpg}-\ac{ppo} improves upon \ac{ppo} trained on \ac{hf} data alone in nearly all settings, and consistently matches the performance of \ac{ppo} trained with $16\times$ more \ac{hf} data on the hardest task at the smallest \ac{hf} budgets. In contrast, most baselines that use \ac{lf} data perform well only where direct \ac{lf}-to-\ac{hf} transfer succeeds. \ac{mfpg}-\ac{ppo} enables stable learning on a physical Franka arm using only 4 real-robot episodes per update and no human demonstrations.

\end{abstract}

\section{Introduction}

On-policy \acf{rl} with policy gradient methods underpins advances in areas ranging from language-model post-training \citep{ouyang2022training} to fine-tuning vision-language-action models \citep{lu2025vla} for robot control. Yet in many real-world settings, target-domain data are expensive
to collect and therefore scarce, e.g., a physical robot may collect only a handful of rollouts per policy update. In such data-scarce regimes, standard on-policy methods can break down because their policy gradient estimates become too noisy to support stable learning.

We study how to enable stable on-policy \ac{rl} in data-scarce regimes, ultimately using as few as 4 real-robot episodes per update, or target-domain data batches up to $512\times$ smaller than in standard GPU-simulation \ac{rl} pipelines~\citep{rudin2022learning}. The key is access to \emph{multi-fidelity} data: we assume that the target environment (a physical system or high-accuracy simulator)
provides limited, expensive \acf{hf} samples, whereas a simplified simulator provides abundant, inexpensive, but biased \acf{lf} samples. Exploiting \ac{lf} data requires care, however, and can easily %
bias policy updates and degrade performance in the target \ac{hf} environment. 

Existing approaches use \ac{lf} data in different ways. Off-dynamics \ac{rl} methods~\citep{eysenbach2021offdynamics,lyu2024cross} learn to correct \ac{lf} rewards for the dynamics mismatch, and sim-and-real co-training~\citep{wei2025empirical} mixes \ac{lf} and \ac{hf} data. However, \emph{reliability} is a major challenge. Most of these methods still directly optimize biased objectives based on \ac{lf} data, so their performance in the target \ac{hf} environment is largely tied to how well policies learned from \ac{lf} data transfer.

The recent \acf{mfpg} framework~\citep{liu2026a} offers a principled alternative. %
Here, \ac{lf} samples are used to construct a control variate~\citep{nelson1987control}---a correlated auxiliary random variable with a known mean---for the \ac{hf} policy gradient estimator. The \ac{mfpg} estimator %
remains unbiased with respect to the \ac{hf} objective, and the policy gradient surrogate variance shrinks as the cross-fidelity data correlation grows. \ac{mfpg} is therefore robust to \ac{lf} data biases.

However, the \ac{mfpg} construction of~\citet{liu2026a} is limited to the classic REINFORCE algorithm~\citep{williams1992simple} and small-scale, single-environment simulation tasks. How \ac{mfpg} interacts with modern actor-critic learning and GPU-parallel simulation~\citep{mittal2025isaaclab,tao2024maniskill3}, and how it translates to real-robot learning, remain unresolved. 
An earlier preprint~\citep{liu2025multi_v1}
directly substituted \acf{ppo} for REINFORCE. The resulting
method did not consistently improve over the REINFORCE
instantiation and sometimes even underperformed single-fidelity
\ac{ppo}. Hence, an effective extension requires more than algorithm substitution.

We address this gap with two contributions: (i) \emph{\ac{mfpg}-\ac{ppo}, a multi-fidelity \ac{ppo} algorithm.} Through analysis and experiments, we show how naive \ac{ppo} instantiations of \ac{mfpg} can lose cross-fidelity correlation or inflate variance by orders of magnitude. We address these failures by redesigning the sampling, advantage estimation, and control variate construction to preserve cross-fidelity correlation, and by monitoring estimator uncertainty to prevent variance inflation. Preserving correlation comes at the price of potential bias, a tradeoff we characterize analytically. (ii) \emph{A budget-aware \ac{mfpg}-\ac{ppo}.} When \ac{hf} and \ac{lf} samples incur different costs, we allocate a fixed per-iteration sampling budget across fidelities, guided by the variance-minimizing allocation framework of \citet{peherstorfer2016optimal}.

\begin{figure}[tb]
    \centering
    \includegraphics[width=0.86\linewidth]{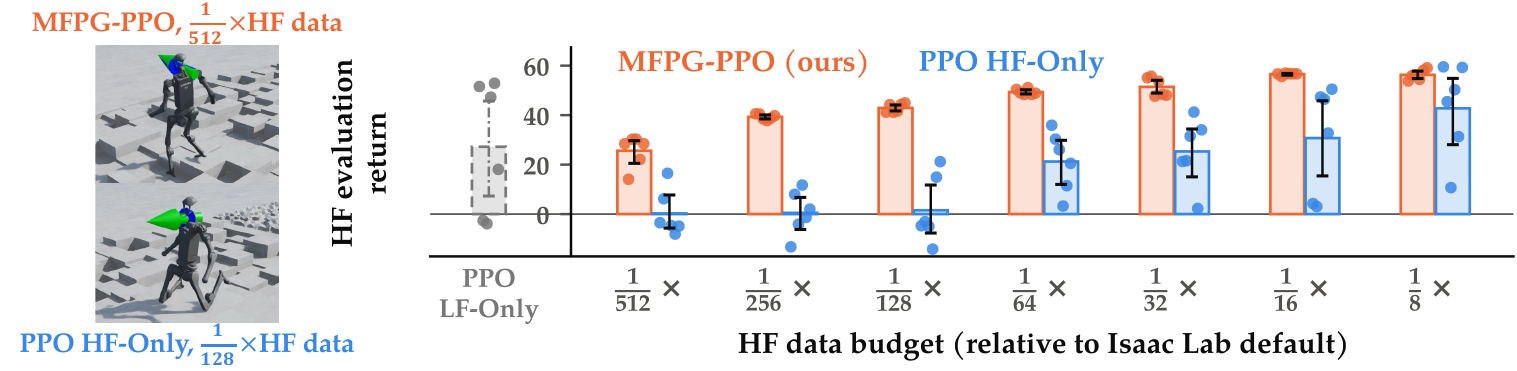}
    \caption{
\textbf{In a simulated humanoid robot locomotion task on rough terrain, \ac{mfpg}-\ac{ppo} (orange) \emph{scales monotonically} with the \ac{hf} data budget and, where \ac{ppo} trained on \ac{hf} data alone (blue) collapses, matches the mean return of \ac{ppo} trained with $16\times$ more \ac{hf} data.} \ac{hf} data budgets are set by the number of parallel \ac{hf} training environments and expressed as fractions of the Isaac Lab default (4096 environments). At each \ac{hf} budget, \ac{mfpg}-\ac{ppo} uses the same amount of \ac{hf} data as \ac{ppo} \ac{hf}-Only, plus abundant \ac{lf} data from a simplified simulator. \ac{ppo} \ac{lf}-Only (gray) trains on abundant \ac{lf} data alone and is evaluated in the \ac{hf} environment. Bars show mean \ac{hf} evaluation returns of final policies across
6 seeds, with $95\%$ seed-bootstrapped confidence intervals.
Dots denote individual seeds.
The snapshots show the \ac{mfpg}-\ac{ppo} policy walking across the terrain, whereas the \ac{ppo} \ac{hf}-Only policy remains on its knees despite using $4\times$ more \ac{hf} training data.
}
    \label{fig:scaling-v3}
\end{figure}

We test \ac{mfpg}-\ac{ppo} in GPU-parallel simulation~\citep{mittal2025isaaclab,tao2024maniskill3} and on a physical robot arm. Across simulation tasks and \ac{hf} data budgets, we compare with 5 baselines that use \ac{lf} data in different ways and with \ac{ppo} and \ac{sac}~\citep{haarnoja2018soft} trained on \ac{hf} data alone. \ac{mfpg}-\ac{ppo} is the only method in our comparison that improves over \ac{ppo} trained on \ac{hf} data alone in nearly all settings, is never distinguishably worse, and scales monotonically with more \ac{hf} data. Most baselines that use \ac{lf} data perform well only where direct \ac{lf}-to-\ac{hf} transfer succeeds. On the hardest task at the smallest \ac{hf} budgets, \ac{mfpg}-\ac{ppo} consistently matches the mean return of \ac{ppo} trained with $16\times$ more \ac{hf} data, cf. \cref{fig:scaling-v3}. Finally, \ac{mfpg}-\ac{ppo} enables stable \ac{rl} on a Franka arm with only 4 episodes per policy update and no human demonstrations.

\section{Related Work}

\ac{lf} simulation can support policy learning for an \ac{hf} target domain through system identification~\citep{ljung1998system,ramos2019bayessim}, zero-shot transfer with domain randomization~\citep{tobin2017domain,peng2018sim}, adapting simulation-trained priors with scarce \ac{hf} data~\citep{smith2022legged,levy2026simulation}, and combining multi-fidelity data within policy updates. We focus on the last category. A complementary line of multi-fidelity \ac{rl} research~\citep{cutler2015real,marco2017virtual,ryou2026multi,bhola2023multi,sifaou2025multi} studies which fidelity to query at each training stage rather than how to combine fidelities within a policy update.

Three mechanisms for combining multi-fidelity data are particularly relevant. \emph{Pooling} multi-fidelity data is a simple and effective recipe for sim-and-real co-training in imitation learning~\citep{maddukuri2025sim,wei2025empirical,lei2026mechanistic} and in \ac{rl}~\citep{di2021sim}. Off-dynamics \ac{rl} \emph{corrects} for estimated dynamics mismatch in \ac{lf} data, e.g., through reward penalties from domain classifiers~\citep{eysenbach2021offdynamics} or representation mismatch~\citep{lyu2024cross}, filtering by value discrepancy~\citep{xu2023cross}, support extension~\citep{van2024policy}, or imitation-based transfer~\citep{guo2024off}, with extensions to offline and hybrid offline-online settings~\citep{liu2022dara,niu2022trust,niu2025h2o+,lyu2025crossdomain,guo2026mobody}.

Pooling and correction optimize biased objectives based on \ac{lf} data, whereas multi-fidelity \emph{control variates} use \ac{lf} data to reduce variance in \ac{hf} estimates. Control variates are a classical tool in single-fidelity \ac{rl}~\citep{greensmith2004variance,gu2017qprop,cheng2020trajectory}. Recently, \citet{khairy2024multi} develop multi-fidelity control variates for state-action value estimation in tabular \acp{mdp}, while \citet{liu2026a} propose a policy gradient framework based on REINFORCE with function approximation. Building on these advances, we develop multi-fidelity control variates for actor-critic learning in GPU-parallel simulation and on a real robot.

\section{Preliminaries}
\label{sec:preliminaries}

\textbf{Multi-fidelity \acp{mdp}.} We consider a \hfc{\emph{\acf{hf}}} and a \lfc{\emph{\acf{lf}}} environment, modeled as infinite-horizon discounted \acp{mdp} $\mdp^f := (\stateSpace, \actionSpace, P^f, r^f, \initialStateDist, \gamma)$ for $f \in \{\high,\low\}$. Both share the state and action spaces $\stateSpace$ and $\actionSpace$, initial-state distribution $\initialStateDist$, and the discount factor $\gamma \in [0,1)$, but may differ in their next-state distributions
$P^f(\cdot \given \state, \action)$ and reward functions
$r^f\colon \stateSpace \times \actionSpace \to \R$,
where $\state\in\stateSpace$ denotes the current state and
$\action\in\actionSpace$ the action. 
We write $\trajectory \sim (\mdp^f, \policy_\policyparams)$ for a trajectory $\trajectory := (\state_0, \action_0, \state_1, \action_1, \dots)$ generated by a stochastic policy $\policy_\policyparams(\action \given \state)$ with parameters $\policyparams \in \R^{\policyParamsDim}$ in $\mdp^f$, and seek the policy that maximizes the expected discounted \ac{hf} return $J^\high(\policyparams) := \E_{\trajectory \sim (\mdp^\high, \policy_\policyparams)}\!\big[\textstyle\sum_{t=0}^{\infty} \gamma^t\, r^\high(\state_t, \action_t)\big]$. We assume the \ac{lf} simulator can be reset to recorded \ac{hf} states for sampling correlated multi-fidelity trajectories, cf. \cref{sec:mfpg-ppo}.

\textbf{Policy gradient surrogates.} Policy gradient methods estimate $\nabla_\policyparams J^\high(\policyparams)$ from trajectories $\trajectory \sim (\mdp^\high, \policy_\policyparams)$. We form estimates by averaging $\nabla_\policyparams \RV(\policyparams)$, where $\RV$ is a scalar surrogate statistic defined per trajectory or transition, depending on the algorithm. REINFORCE~\citep{williams1992simple} uses the \emph{trajectory-level} statistic  $\RV_\trajectory (\policyparams) := \sum_{t=0}^{H-1} \gamma^{t}\, \log \policy_\policyparams(\action_t \given \state_t)\, \big(\hat{G}_t - b(\state_t)\big)$, where $H$ is the episode length, $\hat{G}_t := \sum_{l=t}^{H-1} \gamma^{l-t}\, r^\high(\state_l, \action_l)$ the Monte Carlo return, and $b$ a baseline. During differentiation, $\hat{G}_t$ and $b$ are fixed. Episodic REINFORCE requires collecting complete episodes.

\Acf{ppo}~\citep{schulman2017proximal} operates on \emph{transitions} by replacing returns~$\hat{G}_t$ with advantage estimates
$\hat{A}_t$ commonly computed via \ac{gae}~\citep{schulman2015high}. Using a learned value function (critic) $V$, \ac{gae} sums temporal-difference residuals over a $T$-step rollout within an episode as $\hat{A}_t := \sum_{l=0}^{T-t-1}(\gamma\lambda)^l\delta_{t+l}$, where $\delta_t := r^\high(\state_t,\action_t)+\gamma V(\state_{t+1})-V(\state_t)$ and $\lambda\in[0,1]$ sets the bias-variance tradeoff. $T$ sets the \ac{gae} horizon. Bootstrapping with state values at rollout boundaries allows updates from parallel, short \emph{rollout segments}. \ac{ppo} reuses batches across epochs of minibatch updates on the surrogate $\RV_t(\policyparams) := \min\big(\omega_t \hat{A}_t,\, \operatorname{clip}(\omega_t, 1-\epsilon, 1+\epsilon)\, \hat{A}_t\big)$, where $\omega_t := \policy_\policyparams(\action_t \given \state_t)/\policy_{\policyparams_{\text{old}}}(\action_t \given \state_t)$ is the importance ratio relative to the data-collecting policy and $\epsilon$ the clipping range. These features make \ac{ppo} standard in GPU-based \ac{rl}~\citep{rudin2022learning} but depart from the \ac{mfpg}-REINFORCE setting.

\textbf{Multi-fidelity policy gradients.}\label{sec:tmlr-mfpg} Following~\citet{gu2017qprop} and \citet{liu2026a}, we target the variance of the scalar surrogate statistic $\RV(\policyparams)$ before differentiation. \citet{liu2026a} form a control-variate estimator~\citep{nelson1987control} on REINFORCE's trajectory-level statistic $\RV_\trajectory$ over $N^\high$ \ac{hf} trajectories:
$\nabla_\policyparams J^\high(\policyparams) \approx
(N^\high)^{-1}\sum_{i=1}^{N^\high}
\nabla_\policyparams \big[\hfc{\RV^\high_i(\policyparams)}
+ c\,\big(\lfc{\RV^\low_i(\policyparams)}
- \lfc{\hat{\mu}^\low(\policyparams)}\big)\big]$, where $c \in \R$ is the control variate coefficient and $c\,(\RV^\low_i - \hat{\mu}^\low)$ is a corrector for $\RV^\high_i$. Each $\RV^\low_i$ is computed on the \ac{lf} trajectory \emph{coupled} to its \ac{hf} counterpart: the \ac{lf} simulator is reset to the \ac{hf} initial state and the two rollouts share the same action-noise samples via a reparameterization trick~\citep{kingma2013auto} adapted to \ac{rl}. The mean $\hat{\mu}^\low(\policyparams) := (N^\low)^{-1}\sum_{j=1}^{N^\low} \RV_{\tilde{\trajectory}^\low_j}(\policyparams)$ is estimated from $N^\low \gg N^\high$ additional, independently sampled \ac{lf} trajectories $\tilde{\trajectory}^\low_j$. Intuitively, $\RV^\low$ is an \ac{lf} ``twin'' of the \ac{hf} statistic.
Formally, for $Z(c) := \RV^\high + c\,(\RV^\low - \mu^\low)$ with $\mu^\low := \E[\RV^\low]$, \citet{liu2026a} establish (i) \emph{unbiasedness}, $\E[Z(c)] = \E[\RV^\high]$ for fixed $c$, since the \ac{lf} corrector has zero mean, and (ii) \emph{variance reduction}, with the optimal $c^* = -\operatorname{Cov}(\RV^\high, \RV^\low)/\operatorname{Var}(\RV^\low)$ yielding $\operatorname{Var}\big(Z(c^*)\big) = \big(1-\rho^2\big)\operatorname{Var}(\RV^\high)$, where $\rho$ is the correlation between coupled $\RV^\high$ and $\RV^\low$. Variance reduction grows with $|\rho|$, and $c^*$ is estimated from data across iterations.

\section{Approach}\label{sec:approach}\label{sec:challenges}

We develop \ac{mfpg}-\ac{ppo} by addressing two challenges that arise when extending the control variate construction of \citet{liu2026a} to \ac{ppo}, then introduce a budget-aware \ac{mfpg}-\ac{ppo}.

\textbf{Challenge 1: What to pair as $(\hfc{\RV^\high}, \lfc{\RV^\low})$, and how to keep them correlated.} \citet{liu2026a} pair the complete-episode REINFORCE statistic $\RV_\trajectory$, which is unbiased but has two limitations. First, \ac{ppo} collects short rollouts from desynchronized environments, which need not contain any complete episode. Second, since the multi-fidelity environments have different dynamics, coupled trajectories can diverge over long episodes despite shared initial states and action noise, weakening or destabilizing their correlation $\rho$ on which the variance reduction $(1-\rho^2)$ relies. One trajectory may even terminate while the other continues.

\textbf{Challenge 2: How much to trust the estimates.} The correlation $\rho$ and the coefficient $c^*$ are unknown and estimated from scarce \ac{hf} and coupled \ac{lf} data.
However, a poorly estimated coefficient can make the estimator \emph{noisier} than using \ac{hf} data alone. Moreover, \ac{ppo} reuses each batch across policy updates, so coefficient estimates must account for changes in the surrogate statistics.
The control variate should therefore be applied cautiously only when the evidence supports its use.

\subsection{\ac{mfpg}-\ac{ppo}}\label{sec:mfpg-ppo}

\textbf{Workflow.} Each iteration of \ac{mfpg}-\ac{ppo} follows four steps, cf. \cref{alg:mfpg} in \cref{sec:complete-algo}. First, \emph{multi-fidelity data collection} gathers \ac{hf}, coupled \ac{lf}, and uncoupled \ac{lf} batches of transitions. Within each batch, $k$ indexes independent rollouts, e.g., from parallel environments, and $t$ their steps. Second, \emph{constructing surrogate statistics} computes the \ac{ppo}  statistics defined in \cref{sec:preliminaries} for all three batches. Let $\mathcal{P}$ index the paired statistics $(\hfc{\RV^\high_{k,t}}, \lfc{\RV^\low_{k,t}})$ and $\mathcal{U}$ the uncoupled \ac{lf} statistics. Third, \emph{coefficient fitting} estimates and smooths the coefficient $c$ from the pairs in $\mathcal{P}$ for the estimator
\begin{equation}\label{eq:mfpg-ppo}
    L^{\mathrm{MF}}(\policyparams) := \frac{1}{|\mathcal{P}|}
    \sum_{(k,t) \in \mathcal{P}} \Big[ \hfc{\RV^\high_{k,t}(\policyparams)}
    + c\, \big( \lfc{\RV^\low_{k,t}(\policyparams)} -
    \lfc{\hat{\mu}^\low(\policyparams)} \big) \Big],
    \quad
    \hat{\mu}^\low(\policyparams) := \frac{1}{|\mathcal{U}|}
    \sum_{(k,t) \in \mathcal{U}} \RV^\low_{k,t}(\policyparams).
\end{equation}
The policy gradient estimate is $\nabla_\policyparams L^{\mathrm{MF}}$, with gradients flowing through all three terms, including $\hat{\mu}^\low$, and the advantages and $c$ held fixed. Finally, \emph{monitoring} decides whether to apply the control variate or to fall back to a \ac{ppo} update on \ac{hf} data alone by setting $c = 0$. The first two steps address Challenge~1 and the last two Challenge~2. %

\textbf{Multi-fidelity data collection.} To preserve correlation between the paired statistics $(\hfc{\RV^\high_{k,t}}, \lfc{\RV^\low_{k,t}})$, we limit divergence between \ac{hf} and coupled \ac{lf} rollouts by (i) sharing controllable sources of randomness and (ii) periodically re-synchronizing the coupled \ac{lf} rollouts to recorded \ac{hf} states, cf. lines~\ref{line:collect} to~\ref{line:pairs} of \cref{alg:mfpg}. \Cref{fig:sampling-pipeline} illustrates the multi-fidelity data collection in two settings. 

When \ac{hf} and \ac{lf} environments are GPU-parallel simulations of different accuracy, we collect short $T$-step rollouts from $N^\high$ \ac{hf}, $N^\high$ coupled \ac{lf}, and $N^\low \gg N^\high$ uncoupled \ac{lf} environments. \Cref{fig:sampling-pipeline}(a) illustrates one representative rollout from each batch. We first collect the \ac{hf} rollouts, then reset each coupled \ac{lf} environment to its counterpart's rollout start state, including task variables such as terrain. We assume the policy is \emph{reparameterizable}~\citep{kingma2013auto} as $\action_t^f = g_\policyparams(\state_t^f, \varepsilon_t)$ with a deterministic map $g_\policyparams$ and noise $\varepsilon_t$ drawn from a fixed distribution independent of $\policyparams$. Coupled rollouts share the noise sample at corresponding steps, with the policy evaluated at each fidelity's own state. When one environment of a pair terminates while the other continues, the remaining steps of the continuing environment are uncoupled and excluded from $\mathcal{P}$ until the next re-synchronization, i.e., an environment reset or the start of the next iteration. The rollout length $T$ sets both \ac{ppo}'s rollout horizon and the re-synchronization interval. Uncoupled \ac{lf} environments draw their action noise independently and reset freely to $\state^\low_0 \sim \initialStateDist$.

\begin{figure}[tb]
    \centering
    \includegraphics[width=0.99\linewidth]{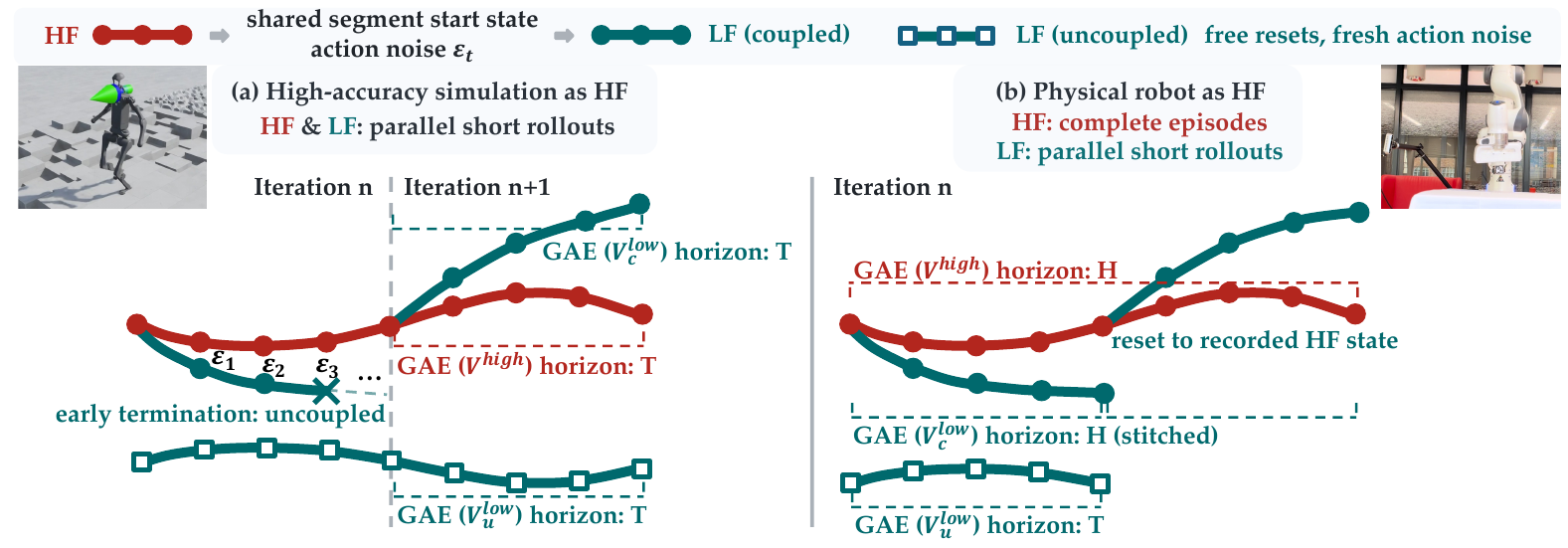}
    \caption{
\textbf{\ac{mfpg}-\ac{ppo} collects three batches of transitions.} Representative \ac{hf}, coupled \ac{lf}, and uncoupled \ac{lf} rollouts are shown with red circles, filled teal circles, and hollow teal squares. \textbf{(a)}~When \ac{hf} and \ac{lf} are simulations of different accuracy, each batch contains parallel $T$-step rollouts per iteration. \textbf{(b)}~When the \ac{hf} environment is a robot collecting $H$-step episodes ($H > T$), coupled \ac{lf} simulation rollouts restart from recorded \ac{hf} states every $T$ steps. In both cases, the \ac{lf} simulation runs only parallel $T$-step rollouts, while the coupled \ac{lf} \acs{gae} horizon matches the \ac{hf} horizon, $T$ in (a) and $H$ in (b). In (b), we stitch the coupled \ac{lf} rollouts of each episode for \acs{gae}.
}
    \label{fig:sampling-pipeline}
\end{figure}

When the \ac{hf} environment is a physical robot and the \ac{lf} environment is simulation, a robot instead collects $K$ complete $H$-step episodes sequentially ($H>T$). We divide each episode into $T$-step segments and launch a coupled \ac{lf} rollout from each segment's \emph{measured} initial state, cf. \cref{fig:sampling-pipeline}(b). These rollouts run in parallel after \ac{hf} collection and re-synchronize frequently to \ac{hf} states.

\textbf{Constructing paired statistics $(\hfc{\RV^\high_{k,t}}, \lfc{\RV^\low_{k,t}})$.} We construct \ac{ppo} statistics by computing \ac{gae} with critics $V^\high$, $V^\low_{\mathrm{c}}$, and $V^\low_{\mathrm{u}}$ for the \ac{hf}, coupled \ac{lf}, and uncoupled \ac{lf} batches, respectively. The separate \ac{lf} critics accommodate the different state distributions of the coupled and uncoupled \ac{lf} rollouts. We match the \ac{gae} horizons of the coupled \ac{lf} rollouts to the \ac{hf} rollouts to preserve cross-fidelity correlation, cf. lines~\ref{line:adv} and~\ref{line:stat} of \cref{alg:mfpg} and \cref{fig:sampling-pipeline}.

In simulation, all three batches use a $T$-step \ac{gae} horizon matching their rollout length, cf. \cref{fig:sampling-pipeline}(a). On a real robot, we match the $H$-step \ac{hf} \ac{gae} horizon by stitching the coupled \ac{lf} rollout segments of the same episode, cf. \cref{fig:sampling-pipeline}(b). Uncoupled \ac{lf} rollouts use an ordinary $T$-step \ac{gae} horizon in both settings. The critics, in contrast, are trained only on transitions their own environment produced, with targets that never span a re-synchronized state, so that each remains a valid value estimate of its own environment. Since the coupled and uncoupled \ac{lf} statistics are computed along different trajectory distributions, the \ac{lf} correction in \cref{eq:mfpg-ppo} can introduce bias, a cost of preserving correlation. We analyze the tradeoff for this design in \cref{sec:free-baseline} and show that re-synchronizing the uncoupled \ac{lf} rollouts as well would remove the bias but leave part of the variance irreducible.

\textbf{Fitting the coefficient $c$.}\label{sec:fitting} We fit $c$ by estimating the optimal value $c^*$ of \cref{sec:preliminaries} from the pairs $(\hfc{\RV^\high_{k,t}}, \lfc{\RV^\low_{k,t}})$, $(k, t) \in \mathcal{P}$, and smooth the estimates with an \ac{ema}. We re-fit $c$ as the policy is updated across minibatches, cf. lines~\ref{line:fit} and~\ref{line:refit} of \cref{alg:mfpg}.

Since the pairs are scarce (Challenge~2), one option, taken by the released implementation of \citet{liu2026a}, treats all pairs $(k, t) \in \mathcal{P}$ as independent samples and fits \emph{per step}, $\hat{c}_{\mathrm{step}} := -\widehat{\operatorname{Cov}}(\RV^\high_{k,t}, \RV^\low_{k,t})/\widehat{\operatorname{Var}}(\RV^\low_{k,t})$. The per-step fit multiplies the sample size and stabilizes the estimate, but it can miss cross-step correlation relevant to variance reduction. On our real-robot manipulation task, cf. \cref{fig:franka-main}(c), key events such as grasping occur asynchronously in the two fidelities, and the per-step correlation is weak ($\hat{\rho} \approx 0.17$). Fitting on the \emph{episode sums} $B^f_k := \sum_t \RV^f_{k,t}$, i.e., $\hat{c}_{\mathrm{ep}} := -\widehat{\operatorname{Cov}}(B^\high_k, B^\low_k)/\widehat{\operatorname{Var}}(B^\low_k)$, instead retains the cross-step correlation ($\hat{\rho} \approx 0.6$). In simulation, per-step correlations are substantially higher, cf. \cref{fig:event_timing} in \cref{app:more-results}. We therefore fit $c$ per step in simulation and on episode sums on the robot. The fitting unit trades off sample size against the cross-step correlation it captures, as formalized in \cref{lem:fit-granularity-vi}, with proof in \cref{app:proof-fit-granularity}.
\begin{lemma}[Per-step fitting omits cross-step covariance]\label{lem:fit-granularity-vi}
    For a fixed policy, consider paired statistics $\big\{ \big( \hfc{\RV^\high_{k,t}},\; \lfc{\RV^\low_{k,t}} \big) \big\}_{k=1,\,t=1}^{K,\;H}$ from $K$ i.i.d. episodes of $H$ steps, with the episode sums $B^\high_k$ and $B^\low_k$ defined above. In the moments below, we drop the episode index $k$. Assume the per-step means $\E[\RV^\high_t]$ and $\E[\RV^\low_t]$ are constant in $t$. Treating $\hat{\mu}^\low$ as fixed and under the regularity conditions of \cref{app:proof-fit-granularity}, the variance of the estimator in \cref{eq:mfpg-ppo} is minimized by $c^* = -\operatorname{Cov}(B^\high, B^\low)/\operatorname{Var}(B^\low)$, where $\operatorname{Cov}(B^\high, B^\low) = \textcolor{violet}{\sum_{t=1}^{H}\sum_{t'=1}^{H}} \operatorname{Cov}(\RV^\high_t, \RV^\low_{t'})$. The per-step fit estimates $c_{\mathrm{step}} = -\textcolor{violet}{\sum_{t}}\operatorname{Cov}(\RV^\high_t, \RV^\low_t)\,/\, \sum_{t}\operatorname{Var}(\RV^\low_t)$, retaining only the same-step terms. Hence, $c_{\mathrm{step}} = c^*$ when cross-step covariances vanish, but not in general.
\end{lemma}

\textbf{Monitoring the control variate.} Since the coefficient $c$ is fitted on scarce data, we apply the control variate only when two checks pass, cf. lines~\ref{line:fit} and~\ref{line:monitor} of \cref{alg:mfpg}. Otherwise, the update falls back to \ac{hf}-only by setting $c = 0$. First, \emph{statistical evidence}. The correlation between the paired statistics must be distinguishable from zero at a chosen confidence level, which we set to $90\%$. In simulation, a bootstrap confidence interval~\citep{efron1992bootstrap} on the fitted $c$ must exclude zero. 
On the real robot, bootstrapping from scarce episode pairs can be unstable. We instead use a heuristic check on the estimated correlation $\hat{\rho}$ between episode sums $(B^\high_k, B^\low_k)$. Since the distribution of $\hat{\rho}$ is skewed, we average its Fisher-$z$ transform $\operatorname{atanh}(\hat{\rho})$~\citep{fisher1921probable}, which is approximately normal, over a sliding window of recent iterations and require the resulting normal confidence interval to exclude zero, cf. \cref{app:gating}. Second, \emph{empirical variance reduction}. The corrected statistic $\hfc{\RV^\high_{k,t}} + c\,\big(\lfc{\RV^\low_{k,t}} - \lfc{\hat{\mu}^\low}\big)$ must have lower sample variance than $\hfc{\RV^\high_{k,t}}$ alone over the pairs.

\subsection{Budget-Aware \ac{mfpg}-\ac{ppo}}\label{sec:budget-allocation}

The estimator in \cref{eq:mfpg-ppo} so far treats \ac{lf} data as effectively free with $N^\low \gg N^\high$. When \ac{lf} sampling is not free, we ask: \emph{given a fixed per-iteration budget and heterogeneous per-sample costs, how many samples should be drawn from each fidelity to minimize the variance of the control variate estimator?} Writing $\samplecosth, \samplecostl$ for the per-sample costs (per $T$-step rollout), $p$ for the budget, and $m^\high, m^\low$ for the \ac{hf} and \emph{total} coupled and uncoupled \ac{lf} sample counts, the question becomes
\begin{equation}\label{eq:budget-alloc}
    \min_{c,\;m^\high,\;m^\low}\; \Var\big[L^{\mathrm{MF}}_{\mathrm{all}}(\policyparams)\big]
    \quad \textrm{s.t.}\quad \samplecosth m^\high + \samplecostl m^\low = p,\quad m^\low \ge m^\high > 0,
\end{equation}
a two-fidelity case of the \ac{mfmc} allocation problem \citep{peherstorfer2016optimal}. Here, $L^{\mathrm{MF}}_{\mathrm{all}}$ is \cref{eq:mfpg-ppo} with the baseline mean $\lfc{\hat{\mu}^\low_{\mathrm{all}}}$ over \emph{all} coupled and uncoupled \ac{lf} samples, as in the \ac{mfmc} estimator. 
Let $\nu := m^\low/m^\high$ denote the ratio of total \ac{lf} to \ac{hf} sample counts. Under the idealized sampling model in \cref{app:budget-allocation}, the variance-minimizing allocation~\citep[Thm.~3.4]{peherstorfer2016optimal} satisfies $\nu^* = \max\big\{1,\, \sqrt{(\samplecosth/\samplecostl)\,\rho^2/(1-\rho^2)}\big\}$ and $m^{\high,*} = p/(\samplecosth+\samplecostl\,\nu^*)$. \Cref{app:budget-allocation} gives a derivation and details on flooring the real-valued solution to integers within the budget. For $\nu^*>1$, the coefficient remains the $c^{*}$ of \cref{sec:tmlr-mfpg}. At $\nu^*=1$, the \ac{lf} correction vanishes. The ratio $\nu^*$ grows as \ac{lf}
rollouts become cheaper or more strongly correlated with
their \ac{hf} counterparts.

\textbf{Online budget allocation.} As in multi-fidelity optimization under uncertainty~\citep{leo2014multifidelity,HYUN2026118714}, we adapt the allocation as the policy changes. We track $\rho$ with the smoothing scheme of \cref{sec:fitting}, re-solve \cref{eq:budget-alloc} every iteration \emph{before} data collection, and set $N^\high=m^{\high,*}$ and $N^\low=m^{\low,*}-m^{\high,*}$ after flooring to integers.  Coefficient fitting and monitoring carry over unchanged, and the policy objective becomes $L^{\mathrm{MF}}_{\mathrm{all}}$. Whenever $\nu^{*} = 1$, i.e., \ac{lf} data do not warrant more samples than \ac{hf} data, the iteration collects only coupled \ac{hf}-\ac{lf} data, which preserves our ability to estimate $\rho$, and the update falls back to \ac{hf}-only \ac{ppo}, cf. \cref{app:budget-allocation}.

\section{Experiments}

\textbf{Evaluation criteria.} A practitioner rarely knows in advance how informative the \ac{lf} data will be. We therefore evaluate two criteria. (i) \emph{Monotone scaling}, i.e., an approach's performance should improve monotonically as more \ac{hf} data become available, until it plateaus. (ii) \emph{Reliability}, i.e., compared to training with limited \ac{hf} data alone, an approach should perform better when the \ac{lf} data are informative and no worse when they are misleading.

\subsection{Experimental Setup}\label{sec:setup}

\textbf{Simulation setup and \ac{hf} data budgets.} We evaluate locomotion on rough and flat terrain with the Unitree H1 humanoid and the ANYmal-D quadruped in Isaac Lab~\citep{mittal2025isaaclab}. We sweep a \emph{sample-scarcity} axis by varying the number of parallel \ac{hf} environments $N^\high$ from 8 to 512, up to $512\times$ fewer than the Isaac Lab default of 4096, while fixing $N^\low=4096$ uncoupled \ac{lf} environments. The \ac{lf} simulation simplifies the robot collision model, coarsens the physics time step, reduces the number of physics solver iterations, and disables self-collisions. 
Since exploration from a random initialization is difficult with scarce \ac{hf} data, all on-policy methods that use \ac{hf} data begin from a shared policy briefly pre-trained with \ac{ppo} in the \ac{lf} environment. \ac{sac} starts from its own \ac{lf}-pretrained policy with matched \ac{lf} training return. Unless otherwise stated, we evaluate each final policy deterministically
over $2048$ \ac{hf} episodes and report mean returns with $95\%$
bootstrap confidence intervals across 6 seeds. \Cref{app:implementation-details} reports details.

\textbf{Real-robot setup.} We evaluate cube picking with a
physical Franka arm and ManiSkill
simulation~\citep{tao2024maniskill3} as the \ac{hf} and \ac{lf} environments, respectively. The task is harder in the \ac{hf} environment because the robot's side-mounted camera loses sight of the object during contact (around 69\% of the steps), and grasp detection relies on gripper width and may be delayed. By contrast, the simulator provides ground-truth object and grasp states.
Friction and controller stiffness also differ between fidelities. 
We run 3 random seeds per real-robot method, each for
$80$ iterations of $4$ episodes. \ac{mfpg}-\ac{ppo} collects extra coupled
\ac{lf} rollouts and $8192$ uncoupled $20$-step \ac{lf} rollouts.

\textbf{Baselines and shared tuning protocol.} We evaluate two \ac{mfpg}-\ac{ppo} variants. \emph{MFPG-PPO-FR} in \cref{sec:mfpg-ppo} uses a \emph{fixed} \ac{hf}-to-\ac{lf} sampling \emph{ratio}, and \emph{MFPG-PPO-BA} in \cref{sec:budget-allocation} performs \emph{budget} \emph{allocation} online. Baselines cover several uses of \ac{lf} data. \emph{\ac{ppo} \ac{hf}-Only} uses no \ac{lf} data. \emph{Co-training}~\citep{wei2025empirical} pools \ac{hf} and \ac{lf} data. \emph{DARC}~\citep{eysenbach2021offdynamics} and \emph{PAR}~\citep{lyu2024cross} correct \ac{lf} rewards. \emph{\ac{ppo} \ac{lf}-Only} uses no \ac{hf} data. \emph{\ac{mfpg}-REINFORCE}~\citep{liu2026a} uses control variates. 
To isolate the effect of \ac{lf} data use from the algorithm backbone, we adapt DARC and PAR to the \ac{ppo} backbone of \ac{mfpg}-\ac{ppo}, following the practice of \citet{liu2026a}, and Co-training to \ac{rl} on the same backbone. We additionally compare to an off-policy baseline \emph{\ac{sac} \ac{hf}-Only}~\citep{haarnoja2018soft}. Each baseline matches the data budget of \ac{mfpg}-\ac{ppo}-FR at each fidelity it consumes. \Cref{app:implementation-details} describes the baselines and the shared hyperparameter tuning protocol. \Cref{app:more-results} reports extended comparisons.

\subsection{How does \ac{mfpg}-\ac{ppo} scale with \ac{hf} data scarcity?}\label{sec:mfpg-scaling}

The bar chart in \cref{fig:scaling-v3} compares MFPG-PPO-FR with \ac{ppo} \ac{hf}-Only at matched \ac{hf} data budgets on H1 Rough, the hardest task in our suite. 
\ac{ppo} \ac{hf}-Only collapses to near-zero return at \ac{hf} budgets of $\frac{1}{128}\times$ and below, whereas MFPG-PPO-FR still attains a mean return around $40$ at $\frac{1}{256}\times$. 
At each of the 3 smallest \ac{hf} budgets,
MFPG-PPO-FR matches or exceeds the mean return of
\ac{ppo} \ac{hf}-Only trained with $16\times$ more \ac{hf} data while exhibiting substantially
lower cross-seed variability. 
\ac{ppo} \ac{lf}-Only transfers poorly on half of the seeds, and even the seeds that transfer are brittle, with roughly twice the cross-episode return variability of \ac{mfpg}-\ac{ppo}-FR.

\subsection{How do different uses of \ac{lf} data scale across tasks and \ac{hf} budgets?}\label{sec:baseline-comparison}

\Cref{fig:baseline-delta} reports returns relative to \ac{ppo} \ac{hf}-Only across tasks and \ac{hf} budgets from $\frac{1}{64}\times$ to $\frac{1}{8}\times$ the Isaac Lab default. \ac{lf}-to-\ac{hf} transfer is \emph{easy} on ANYmal-D, where \ac{ppo} \ac{lf}-Only matches the best methods, but \emph{hard} on H1, where transfer is poor and seed-sensitive. Baselines that exploit \ac{lf} data greedily follow this pattern: DARC and Co-training achieve among the largest gains on ANYmal-D but fall below or provide no distinguishable gain over \ac{ppo} \ac{hf}-Only on H1. Both \ac{mfpg}-REINFORCE and \ac{sac} \ac{hf}-Only generally outperform \ac{ppo} \ac{hf}-Only on H1 Rough but underperform in most settings on the other 3 tasks.

\begin{figure}[tb]
    \centering
    \includegraphics[width=0.96\linewidth]{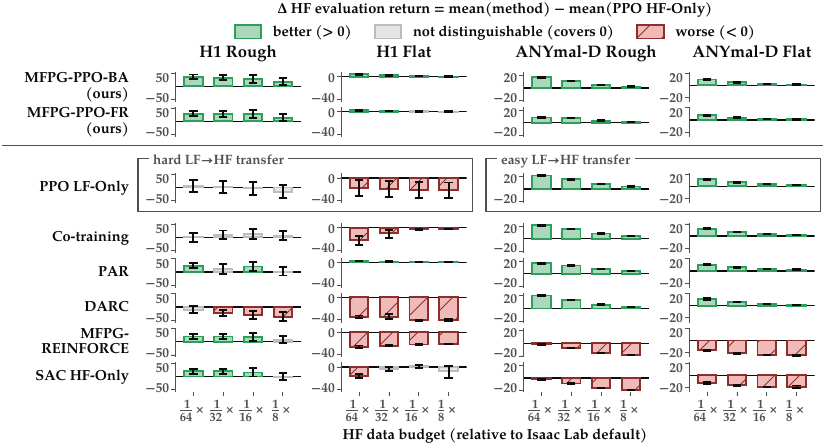}
    \caption{
\textbf{\ac{mfpg}-\ac{ppo} is the only method in our comparison that improves upon \ac{ppo} \ac{hf}-Only in nearly every setting, is never distinguishably worse, and scales monotonically with more \ac{hf} data.} Bars denote differences in mean \ac{hf} evaluation return relative to \ac{ppo} \ac{hf}-Only. Error bars denote
$95\%$ bootstrapped confidence intervals for the mean of the difference across 6 seeds. 
}
    \label{fig:baseline-delta}
\end{figure}

\ac{mfpg}-\ac{ppo} leverages \ac{lf} data reliably. MFPG-PPO-FR outperforms \ac{ppo} \ac{hf}-Only in
$14$ of the $16$ task-budget settings and is never
distinguishably worse. The remaining 2, on H1 Flat, are not distinguishable because \ac{ppo} \ac{hf}-Only already performs among the best. MFPG-PPO-BA, assuming a cost ratio $\samplecosth/\samplecostl = 10$, outperforms \ac{ppo} \ac{hf}-Only in all $16$ settings. Its \ac{hf} budget labels are nominal, since it allocates a sampling budget online. The gains of \ac{mfpg}-\ac{ppo} are largest on the hardest task, H1 Rough, at the smallest \ac{hf} budgets. This reliability comes with a tradeoff. In easy \ac{lf}-to-\ac{hf} settings, the greedier methods gain more from the \ac{lf} data.

PAR also improves in $14$ settings and is never
distinguishably worse than \ac{ppo} \ac{hf}-Only, but it scales less consistently with \ac{hf} data. On H1 Rough, its mean return decreases at 2 of the 3 budget increases. MFPG-PPO-FR's mean return instead rises with more \ac{hf} data in every task, plateauing only at the largest budget on H1 Rough, as shown in \cref{fig:baseline-absolute-perform} in \cref{app:more-results}.

\subsection{Can \ac{mfpg}-\ac{ppo} learn stably on a real robot?}\label{sec:franka}

\begin{figure}[tb]
    \centering
    \includegraphics[width=0.85\linewidth]{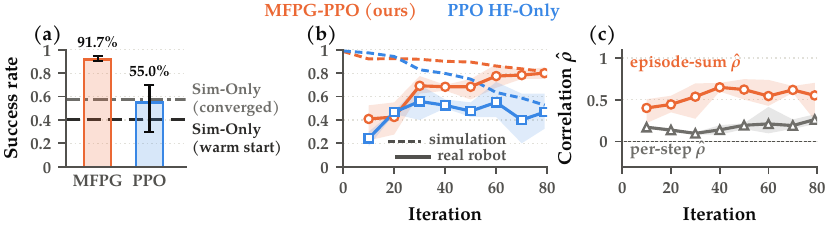}
    \caption{
\textbf{Neither limited real-world data nor abundant simulation data alone yield a good real-robot policy, whereas \ac{mfpg}-\ac{ppo} succeeds by combining them.} Bars and curves denote means across 3 seeds. Whiskers and shaded bands show their ranges. 
\textbf{(a)}~Final success rate of the deterministic policy over $60$ real-robot episodes ($20$ per seed).
Horizontal lines denote
the shared warm-start checkpoint performance and that of a simulation-trained policy well past convergence ($100\%$ simulation success),
each evaluated over $40$ real-robot episodes. \textbf{(b)}~Training success rate on the robot
(solid lines denote running averages over 40 episodes per seed)
and in uncoupled \ac{lf} simulation (dashed lines). \textbf{(c)}~Pearson correlation between the coupled \ac{lf} and \ac{hf} paired statistics, computed on episode sums (4 pairs) and per step (240 pairs), averaged over the past 10 iterations. 
        }
    \label{fig:franka-main}
\end{figure}

We test \ac{mfpg}-\ac{ppo} on the Franka cube-picking task in \cref{sec:setup}. Simulation data are near-free relative to real-robot data, so we use \ac{mfpg}-\ac{ppo}-FR. We compare it with \ac{ppo} \ac{hf}-Only, initializing both from the same simulation-trained policy, which has near-perfect simulation success but retains exploration noise. \Cref{fig:franka-main}(a) reports the final evaluation results. Starting from $40.0\%$ real-robot success, MFPG-PPO-FR reaches $91.7\%$, compared with $55.0\%$ for \ac{ppo} \ac{hf}-Only, and every MFPG-PPO-FR seed ($90$--$95\%$) exceeds every \ac{ppo} \ac{hf}-Only seed ($30$--$70\%$). \Cref{fig:franka-main}(b) shows a consistent pattern during training: the real-robot success rate of MFPG-PPO-FR trends above \ac{ppo} \ac{hf}-Only after 30 iterations. Meanwhile, the same policies' success rates in simulation, cf. the dashed lines in \cref{fig:franka-main}(b), decline toward their real-robot performance, consistent with adaptation to the hardware.

\Cref{fig:franka-main}(c) shows the fitting-unit tradeoff for the coefficient $c$ in \cref{lem:fit-granularity-vi}. The same policy tends to grasp earlier in simulation than in the real world, cf. \cref{fig:event_timing} in \cref{app:more-results}. On the same coupled samples, the episode-sum correlation is higher but fluctuates more than the per-step correlation.

\subsection{Ablations: naive control variate construction amplifies variance}\label{sec:ablation}

\begin{figure}[tb]
    \centering
    \includegraphics[width=0.85\linewidth]{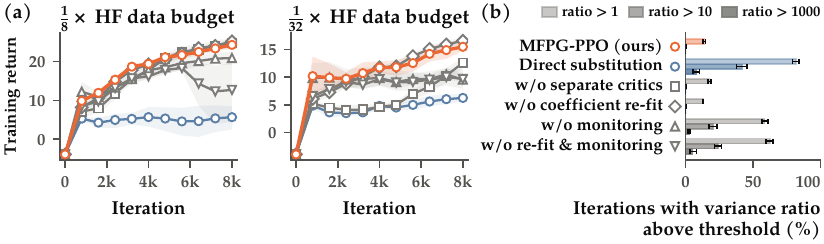}
    \caption{
Curves and bars denote means over 6 seeds. Bands and whiskers denote seed-bootstrapped $95\%$ confidence intervals for the means. \textbf{(a)}~\ac{hf} training return on H1 Rough. \textbf{(b)}~Fraction of training iterations whose variance ratio $\Var[\hfc{\RV^\high_{k,t}}
+ c\,(\lfc{\RV^\low_{k,t}}-\lfc{\hat{\mu}^\low})]
\,/\,\Var[\hfc{\RV^\high_{k,t}}]$, averaged over minibatches \emph{before} the monitoring check, exceeds $1$, $10$, and $1000$ at the $\frac{1}{8}\times$ \ac{hf} budget. 
    }
    \label{fig:ablation}
\end{figure}

We compare MFPG-PPO-FR with component ablations and Direct substitution, which substitutes \ac{ppo} for REINFORCE in \ac{mfpg}, mirroring the preprint~\citep{liu2025multi_v1}. Direct substitution lacks the estimator-side machinery of \cref{sec:mfpg-ppo} (separate critics, monitoring, per-minibatch coefficient re-fitting) but shares our data-collection scheme. In \cref{fig:ablation}(a), Direct substitution yields the lowest training return at both \ac{hf} budgets on H1 Rough. Per-minibatch coefficient re-fitting and monitoring act \emph{jointly}. Removing either alone at most lowers the return, while removing both causes learning to collapse at the $\frac{1}{8}\times$ budget. Removing separate critics has little effect at $\frac{1}{8}\times$ but lowers returns at $\frac{1}{32}\times$. \Cref{fig:ablation}(b) shows that all variants without monitoring \emph{amplify} variance in most iterations (variance ratio above $1$), with ratios sometimes exceeding $1000$. For the full \ac{mfpg}-\ac{ppo}, the ratio exceeds $1$ in only a small fraction of iterations. Minibatches failing the variance reduction check fall back to \ac{hf}-only \ac{ppo} updates. \Cref{app:more-results} reports further diagnostics.

\section{Conclusions and Discussion}\label{sec:conclusions}

We introduced \ac{mfpg}-\ac{ppo} and its budget-aware variant to stabilize data-scarce on-policy \ac{rl}. Through designs that preserve cross-fidelity correlation and monitor estimator uncertainty, \ac{mfpg}-\ac{ppo} enables stable learning with per-update target-domain data batches as small as 4 real-robot episodes or up to $512\times$ smaller than in common GPU-simulation \ac{rl} pipelines~\citep{rudin2022learning}. Unlike methods that directly optimize biased objectives based on \ac{lf} data, \ac{mfpg}-\ac{ppo} uses \ac{lf} data to construct a variance-reducing control variate. In our comparison across tasks and \ac{hf} budgets, \ac{mfpg}-\ac{ppo} is the only method that improves upon \ac{ppo} trained on \ac{hf} data alone in nearly all settings, is never distinguishably worse, and scales monotonically with more \ac{hf} data.

Three directions call for future research. First, re-synchronizing coupled \ac{lf} rollouts to recorded \ac{hf} states while uncoupled \ac{lf} rollouts evolve freely can introduce bias in \cref{eq:mfpg-ppo} if the two \ac{lf} statistics differ in mean. Re-synchronizing both \ac{lf} batches would remove this bias but leave the variance associated with the re-synchronization states \emph{irreducible}, cf. \cref{lem:sync-baseline} in \cref{sec:free-baseline}. Our construction keeps the full variance reducible, but doing so without bias remains an open challenge. Second, raising cross-fidelity correlation, e.g., through representation alignment~\citep{lei2026mechanistic}, could strengthen variance reduction. Third, real-to-sim tools that turn scene scans into interactive simulators~\citep{jain2025polaris,worldlabs2026atlas} could ease the construction of coupled \ac{lf} simulators.

\subsection*{AI use statement}

In this work, we used generative AI tools to implement methods, clean and reformat datasets, interpret results, assist in the writing of proofs, and design or provide feedback on research methodology or experiments. Additionally, we used generative AI tools for sourcing/searching for information and to create or modify scientific figures or images, summarize or analyze existing literature, edit this paper to improve readability, draft parts of this paper, and suggest a structure for this paper. 

We did not use generative AI tools to generate synthetic data sets, help develop theoretical models or conceptual frameworks, formulate mathematical claims, provide critical ingredients for proving mathematical claims, propose or refine hypotheses, or discover research topics or identify gaps. Using generative AI tools to assist with translation, formulate questions for surveys or interviews, and support qualitative and thematic data analysis is not applicable to this work. 

We have reviewed all AI-assisted work. In particular, we reviewed and tested all AI-written code, checked every AI-assisted proof step, and verified all reported experimental results against the experiment logs. We take responsibility for the final content of this work, including text, claims, and artifacts produced with the aid of generative AI.

\subsection*{Reproducibility statement}

\Cref{alg:mfpg} presents the complete \ac{mfpg}-\ac{ppo} algorithm, and \cref{app:gating} specifies its monitoring checks for the control variate estimator. \Cref{sec:hyperparam} reports the shared hyperparameter tuning protocol and the final hyperparameters, \cref{app:method-implementation} describes the implementation of \ac{mfpg}-\ac{ppo} and all baselines, and \cref{app:experimental-protocol} details the simulation and real-robot tasks, their \ac{hf} and \ac{lf} environments, and the training and evaluation protocols. \Cref{app:proof-fit-granularity} proves \cref{lem:fit-granularity-vi}, \cref{sec:free-baseline} analyzes why the uncoupled \ac{lf} rollouts are not re-synchronized, and \cref{app:budget-allocation} derives the budget allocation. We will release our software implementation upon publication.

\subsection*{Acknowledgments}
This work was supported in part by the National Science Foundation (NSF) 
under Grants 2409535, 2535892, and 2214939, and in part by 
the Office of Naval Research (ONR) under Grant ONR N00014-25-1-2479.

\bibliography{iclr2027_conference}

\begin{thebibliography}{54}
\providecommand{\natexlab}[1]{#1}
\providecommand{\url}[1]{\texttt{#1}}
\expandafter\ifx\csname urlstyle\endcsname\relax
  \providecommand{\doi}[1]{doi: #1}\else
  \providecommand{\doi}{doi: \begingroup \urlstyle{rm}\Url}\fi

\bibitem[Bharthulwar et~al.(2026)Bharthulwar, Tao, and Su]{bharthulwar2026staggered}
Sid Bharthulwar, Stone Tao, and Hao Su.
\newblock Staggered environment resets improve massively parallel on-policy reinforcement learning.
\newblock \emph{Advances in Neural Information Processing Systems}, 38:\penalty0 133342--133375, 2026.

\bibitem[Bhola et~al.(2023)Bhola, Pawar, Balaprakash, and Maulik]{bhola2023multi}
Sahil Bhola, Suraj Pawar, Prasanna Balaprakash, and Romit Maulik.
\newblock Multi-fidelity reinforcement learning framework for shape optimization.
\newblock \emph{Journal of Computational Physics}, 482:\penalty0 112018, 2023.

\bibitem[Cheng et~al.(2020)Cheng, Yan, and Boots]{cheng2020trajectory}
Ching-An Cheng, Xinyan Yan, and Byron Boots.
\newblock Trajectory-wise control variates for variance reduction in policy gradient methods.
\newblock In \emph{Conference on Robot Learning}, pp.\  1379--1394. PMLR, 2020.

\bibitem[Cutler et~al.(2015)Cutler, Walsh, and How]{cutler2015real}
Mark Cutler, Thomas~J. Walsh, and Jonathan~P. How.
\newblock Real-world reinforcement learning via multifidelity simulators.
\newblock \emph{IEEE Transactions on Robotics}, 31\penalty0 (3):\penalty0 655--671, 2015.
\newblock \doi{10.1109/TRO.2015.2419431}.

\bibitem[Di-Castro et~al.(2021)Di-Castro, Di~Castro, and Mannor]{di2021sim}
Shirli Di-Castro, Dotan Di~Castro, and Shie Mannor.
\newblock Sim and real: Better together.
\newblock \emph{Advances in Neural Information Processing Systems}, 34:\penalty0 6868--6880, 2021.

\bibitem[Efron(1992)]{efron1992bootstrap}
Bradley Efron.
\newblock Bootstrap methods: another look at the jackknife.
\newblock In \emph{Breakthroughs in statistics: Methodology and distribution}, pp.\  569--593. Springer, 1992.

\bibitem[Eysenbach et~al.(2021)Eysenbach, Chaudhari, Asawa, Levine, and Salakhutdinov]{eysenbach2021offdynamics}
Benjamin Eysenbach, Shreyas Chaudhari, Swapnil Asawa, Sergey Levine, and Ruslan Salakhutdinov.
\newblock Off-dynamics reinforcement learning: Training for transfer with domain classifiers.
\newblock In \emph{International Conference on Learning Representations}, 2021.
\newblock URL \url{https://openreview.net/forum?id=eqBwg3AcIAK}.

\bibitem[Fisher(1921)]{fisher1921probable}
Ronald~A Fisher.
\newblock On the" probable error" of a coefficient of correlation deduced from a small sample.
\newblock \emph{Metron}, 1:\penalty0 3--32, 1921.

\bibitem[Greensmith et~al.(2004)Greensmith, Bartlett, and Baxter]{greensmith2004variance}
Evan Greensmith, Peter~L Bartlett, and Jonathan Baxter.
\newblock Variance reduction techniques for gradient estimates in reinforcement learning.
\newblock \emph{Journal of Machine Learning Research}, 5\penalty0 (Nov):\penalty0 1471--1530, 2004.

\bibitem[Gu et~al.(2017)Gu, Lillicrap, Ghahramani, Turner, and Levine]{gu2017qprop}
Shixiang Gu, Timothy Lillicrap, Zoubin Ghahramani, Richard~E. Turner, and Sergey Levine.
\newblock Q-prop: Sample-efficient policy gradient with an off-policy critic.
\newblock In \emph{International Conference on Learning Representations}, 2017.
\newblock URL \url{https://openreview.net/forum?id=SJ3rcZcxl}.

\bibitem[Guo et~al.(2024)Guo, Wang, Shi, Xu, and Liu]{guo2024off}
Yihong Guo, Yixuan Wang, Yuanyuan Shi, Pan Xu, and Anqi Liu.
\newblock Off-dynamics reinforcement learning via domain adaptation and reward augmented imitation.
\newblock \emph{Advances in Neural Information Processing Systems}, 37:\penalty0 136326--136360, 2024.

\bibitem[Guo et~al.(2026)Guo, Yang, Xu, and Liu]{guo2026mobody}
Yihong Guo, Yu~Yang, Pan Xu, and Anqi Liu.
\newblock Mobody: Model-based off-dynamics offline reinforcement learning.
\newblock In \emph{International Conference on Learning Representations}, volume 2026, pp.\  6985--7008, 2026.

\bibitem[Haarnoja et~al.(2018)Haarnoja, Zhou, Abbeel, and Levine]{haarnoja2018soft}
Tuomas Haarnoja, Aurick Zhou, Pieter Abbeel, and Sergey Levine.
\newblock Soft actor-critic: Off-policy maximum entropy deep reinforcement learning with a stochastic actor.
\newblock In \emph{International conference on machine learning}, pp.\  1861--1870. Pmlr, 2018.

\bibitem[Hessel et~al.(2018)Hessel, Modayil, Van~Hasselt, Schaul, Ostrovski, Dabney, Horgan, Piot, Azar, and Silver]{hessel2018rainbow}
Matteo Hessel, Joseph Modayil, Hado Van~Hasselt, Tom Schaul, Georg Ostrovski, Will Dabney, Dan Horgan, Bilal Piot, Mohammad Azar, and David Silver.
\newblock Rainbow: Combining improvements in deep reinforcement learning.
\newblock In \emph{Proceedings of the AAAI conference on artificial intelligence}, volume~32, 2018.

\bibitem[Hyun et~al.(2026)Hyun, Chaudhuri, Willcox, and Kim]{HYUN2026118714}
Jaeyub Hyun, Anirban Chaudhuri, Karen~E. Willcox, and H.~Alicia Kim.
\newblock Robust topology optimization of continuum structures under material uncertainties using multifidelity monte carlo.
\newblock \emph{Computer Methods in Applied Mechanics and Engineering}, 451:\penalty0 118714, 2026.
\newblock ISSN 0045-7825.
\newblock \doi{https://doi.org/10.1016/j.cma.2025.118714}.
\newblock URL \url{https://www.sciencedirect.com/science/article/pii/S0045782525009867}.

\bibitem[Jain et~al.(2025)Jain, Zhang, Arora, Chen, Torne, Irshad, Zakharov, Wang, Levine, Finn, et~al.]{jain2025polaris}
Arhan Jain, Mingtong Zhang, Kanav Arora, William Chen, Marcel Torne, Muhammad~Zubair Irshad, Sergey Zakharov, Yue Wang, Sergey Levine, Chelsea Finn, et~al.
\newblock Polaris: Scalable real-to-sim evaluations for generalist robot policies.
\newblock \emph{arXiv preprint arXiv:2512.16881}, 2025.

\bibitem[Khairy \& Balaprakash(2024)Khairy and Balaprakash]{khairy2024multi}
Sami Khairy and Prasanna Balaprakash.
\newblock Multi-fidelity reinforcement learning with control variates.
\newblock \emph{Neurocomputing}, 597:\penalty0 127963, 2024.

\bibitem[Kingma \& Welling(2013)Kingma and Welling]{kingma2013auto}
Diederik~P Kingma and Max Welling.
\newblock Auto-encoding variational bayes.
\newblock \emph{arXiv preprint arXiv:1312.6114}, 2013.

\bibitem[Le~Pham~Van et~al.(2024)Le~Pham~Van, The~Tran, and Gupta]{van2024policy}
Linh Le~Pham~Van, Hung The~Tran, and Sunil Gupta.
\newblock Policy learning for off-dynamics rl with deficient support.
\newblock In \emph{Proceedings of the 23rd International Conference on Autonomous Agents and Multiagent Systems}, AAMAS '24, pp.\  1093–1100, Richland, SC, 2024. International Foundation for Autonomous Agents and Multiagent Systems.
\newblock ISBN 9798400704864.

\bibitem[Lei et~al.(2026)Lei, Liu, Maddukuri, Jiang, and Zhu]{lei2026mechanistic}
Yu~Lei, Minghuan Liu, Abhiram Maddukuri, Zhenyu Jiang, and Yuke Zhu.
\newblock A mechanistic analysis of sim-and-real co-training in generative robot policies.
\newblock \emph{arXiv preprint arXiv:2604.13645}, 2026.

\bibitem[Levy et~al.(2026)Levy, Westenbroek, Huang, Palafox, Yin, Omidshafiei, Kim, Gupta, and Fridovich-Keil]{levy2026simulation}
Jacob Levy, Tyler Westenbroek, Kevin Huang, Fernando Palafox, Patrick Yin, Shayegan Omidshafiei, Dong-Ki Kim, Abhishek Gupta, and David Fridovich-Keil.
\newblock Simulation distillation: Pretraining world models in simulation for rapid real-world adaptation.
\newblock \emph{arXiv preprint arXiv:2603.15759}, 2026.

\bibitem[Liu et~al.(2022)Liu, Hongyin, and Wang]{liu2022dara}
Jinxin Liu, Zhang Hongyin, and Donglin Wang.
\newblock {DARA}: Dynamics-aware reward augmentation in offline reinforcement learning.
\newblock In \emph{International Conference on Learning Representations}, 2022.
\newblock URL \url{https://openreview.net/forum?id=9SDQB3b68K}.

\bibitem[Liu et~al.(2025)Liu, Neary, Gupta, Ellis, Topcu, and Fridovich-Keil]{liu2025multi_v1}
Xinjie Liu, Cyrus Neary, Kushagra Gupta, Christian Ellis, Ufuk Topcu, and David Fridovich-Keil.
\newblock Multi-fidelity policy gradient algorithms.
\newblock \emph{arXiv preprint arXiv:2503.05696v1}, 2025.
\newblock URL \url{https://arxiv.org/abs/2503.05696v1}.

\bibitem[Liu et~al.(2026)Liu, Neary, Gupta, Suttle, Ellis, ufuk topcu, and Fridovich-Keil]{liu2026a}
Xinjie Liu, Cyrus Neary, Kushagra Gupta, Wesley~A. Suttle, Christian Ellis, ufuk topcu, and David Fridovich-Keil.
\newblock A multi-fidelity control variate approach for policy gradient estimation.
\newblock \emph{Transactions on Machine Learning Research}, 2026.
\newblock ISSN 2835-8856.
\newblock URL \url{https://openreview.net/forum?id=zAo0L7Dcqt}.
\newblock J2C Certification.

\bibitem[Ljung(1998)]{ljung1998system}
Lennart Ljung.
\newblock System identification.
\newblock In \emph{Signal analysis and prediction}, pp.\  163--173. Springer, 1998.

\bibitem[Lu et~al.(2025)Lu, Guo, Zhang, Zhou, Jiang, Gao, Tang, and Wang]{lu2025vla}
Guanxing Lu, Wenkai Guo, Chubin Zhang, Yuheng Zhou, Haonan Jiang, Zifeng Gao, Yansong Tang, and Ziwei Wang.
\newblock Vla-rl: Towards masterful and general robotic manipulation with scalable reinforcement learning.
\newblock \emph{arXiv preprint arXiv:2505.18719}, 2025.

\bibitem[Lyu et~al.(2024)Lyu, Bai, Yang, Lu, and Li]{lyu2024cross}
Jiafei Lyu, Chenjia Bai, Jing-Wen Yang, Zongqing Lu, and Xiu Li.
\newblock Cross-domain policy adaptation by capturing representation mismatch.
\newblock In Ruslan Salakhutdinov, Zico Kolter, Katherine Heller, Adrian Weller, Nuria Oliver, Jonathan Scarlett, and Felix Berkenkamp (eds.), \emph{Proceedings of the 41st International Conference on Machine Learning}, volume 235 of \emph{Proceedings of Machine Learning Research}, pp.\  33638--33663. PMLR, 21--27 Jul 2024.
\newblock URL \url{https://proceedings.mlr.press/v235/lyu24a.html}.

\bibitem[Lyu et~al.(2025)Lyu, Yan, Qiao, Liu, Ma, Ye, Yang, Lu, and Li]{lyu2025crossdomain}
Jiafei Lyu, Mengbei Yan, Zhongjian Qiao, Runze Liu, Xiaoteng Ma, Deheng Ye, Jing-Wen Yang, Zongqing Lu, and Xiu Li.
\newblock Cross-domain offline policy adaptation with optimal transport and dataset constraint.
\newblock In \emph{The Thirteenth International Conference on Learning Representations}, 2025.
\newblock URL \url{https://openreview.net/forum?id=LRrbD8EZJl}.

\bibitem[Maddukuri et~al.(2025)Maddukuri, Jiang, Chen, Nasiriany, Xie, Fang, Huang, Wang, Xu, Chernyadev, Reed, Goldberg, Mandlekar, Fan, and Zhu]{maddukuri2025sim}
Abhiram Maddukuri, Zhenyu Jiang, Lawrence~Yunliang Chen, Soroush Nasiriany, Yuqi Xie, Yu~Fang, Wenqi Huang, Zu~Wang, Zhenjia Xu, Nikita Chernyadev, Scott Reed, Ken Goldberg, Ajay Mandlekar, Linxi Fan, and Yuke Zhu.
\newblock {Sim-and-Real Co-Training: A Simple Recipe for Vision-Based Robotic Manipulation}.
\newblock In \emph{Proceedings of Robotics: Science and Systems}, LosAngeles, CA, USA, June 2025.
\newblock \doi{10.15607/RSS.2025.XXI.109}.

\bibitem[Marco et~al.(2017)Marco, Berkenkamp, Hennig, Schoellig, Krause, Schaal, and Trimpe]{marco2017virtual}
Alonso Marco, Felix Berkenkamp, Philipp Hennig, Angela~P Schoellig, Andreas Krause, Stefan Schaal, and Sebastian Trimpe.
\newblock Virtual vs. real: Trading off simulations and physical experiments in reinforcement learning with bayesian optimization.
\newblock In \emph{2017 IEEE International Conference on Robotics and Automation (ICRA)}, pp.\  1557--1563. IEEE, 2017.

\bibitem[Mittal et~al.(2025)Mittal, Roth, Tigue, Richard, Zhang, Du, Serrano-Muñoz, Yao, Zurbrügg, Rudin, Wawrzyniak, Rakhsha, Denzler, Heiden, Borovicka, Ahmed, Akinola, Anwar, Carlson, Feng, Garg, Gasoto, Gulich, Guo, Gussert, Hansen, Kulkarni, Li, Liu, Makoviychuk, Malczyk, Mazhar, Moghani, Murali, Noseworthy, Poddubny, Ratliff, Rehberg, Schwarke, Singh, Smith, Tang, Thaker, Trepte, Wyk, Yu, Millane, Ramasamy, Steiner, Subramanian, Volk, Chen, Jawale, Kuruttukulam, Lin, Mandlekar, Patzwaldt, Welsh, Zhao, Anes, Lafleche, Moënne-Loccoz, Park, Stepinski, Gelder, Amevor, Carius, Chang, Chen, de~Heras~Ciechomski, Daviet, Mohajerani, von Muralt, Reutskyy, Sauter, Schirm, Shi, Terdiman, Vilella, Widmer, Yeoman, Chen, Grizan, Li, Li, Smith, Wiltz, Alexis, Chang, Chu, Fan, Farshidian, Handa, Huang, Hutter, Narang, Pouya, Sheng, Zhu, Macklin, Moravanszky, Reist, Guo, Hoeller, and State]{mittal2025isaaclab}
Mayank Mittal, Pascal Roth, James Tigue, Antoine Richard, Octi Zhang, Peter Du, Antonio Serrano-Muñoz, Xinjie Yao, René Zurbrügg, Nikita Rudin, Lukasz Wawrzyniak, Milad Rakhsha, Alain Denzler, Eric Heiden, Ales Borovicka, Ossama Ahmed, Iretiayo Akinola, Abrar Anwar, Mark~T. Carlson, Ji~Yuan Feng, Animesh Garg, Renato Gasoto, Lionel Gulich, Yijie Guo, M.~Gussert, Alex Hansen, Mihir Kulkarni, Chenran Li, Wei Liu, Viktor Makoviychuk, Grzegorz Malczyk, Hammad Mazhar, Masoud Moghani, Adithyavairavan Murali, Michael Noseworthy, Alexander Poddubny, Nathan Ratliff, Welf Rehberg, Clemens Schwarke, Ritvik Singh, James~Latham Smith, Bingjie Tang, Ruchik Thaker, Matthew Trepte, Karl~Van Wyk, Fangzhou Yu, Alex Millane, Vikram Ramasamy, Remo Steiner, Sangeeta Subramanian, Clemens Volk, CY~Chen, Neel Jawale, Ashwin~Varghese Kuruttukulam, Michael~A. Lin, Ajay Mandlekar, Karsten Patzwaldt, John Welsh, Huihua Zhao, Fatima Anes, Jean-Francois Lafleche, Nicolas Moënne-Loccoz, Soowan Park, Rob Stepinski, Dirk~Van Gelder, Chris Amevor, Jan Carius, Jumyung Chang, Anka~He Chen, Pablo de~Heras~Ciechomski, Gilles Daviet, Mohammad Mohajerani, Julia von Muralt, Viktor Reutskyy, Michael Sauter, Simon Schirm, Eric~L. Shi, Pierre Terdiman, Kenny Vilella, Tobias Widmer, Gordon Yeoman, Tiffany Chen, Sergey Grizan, Cathy Li, Lotus Li, Connor Smith, Rafael Wiltz, Kostas Alexis, Yan Chang, David Chu, Linxi~"Jim" Fan, Farbod Farshidian, Ankur Handa, Spencer Huang, Marco Hutter, Yashraj Narang, Soha Pouya, Shiwei Sheng, Yuke Zhu, Miles Macklin, Adam Moravanszky, Philipp Reist, Yunrong Guo, David Hoeller, and Gavriel State.
\newblock Isaac lab: A gpu-accelerated simulation framework for multi-modal robot learning.
\newblock \emph{arXiv preprint arXiv:2511.04831}, 2025.
\newblock URL \url{https://arxiv.org/abs/2511.04831}.

\bibitem[Nelson(1987)]{nelson1987control}
Barry~L Nelson.
\newblock On control variate estimators.
\newblock \emph{Computers \& Operations Research}, 14\penalty0 (3):\penalty0 219--225, 1987.

\bibitem[Ng \& Willcox(2014)Ng and Willcox]{leo2014multifidelity}
Leo W.~T. Ng and Karen~E. Willcox.
\newblock Multifidelity approaches for optimization under uncertainty.
\newblock \emph{International Journal for Numerical Methods in Engineering}, 100\penalty0 (10):\penalty0 746--772, 2014.
\newblock \doi{https://doi.org/10.1002/nme.4761}.
\newblock URL \url{https://onlinelibrary.wiley.com/doi/abs/10.1002/nme.4761}.

\bibitem[Niu et~al.(2022)Niu, sharma, Qiu, Li, Zhou, HU, and Zhan]{niu2022trust}
Haoyi Niu, shubham sharma, Yiwen Qiu, Ming Li, Guyue Zhou, Jianming HU, and Xianyuan Zhan.
\newblock When to trust your simulator: Dynamics-aware hybrid offline-and-online reinforcement learning.
\newblock In S.~Koyejo, S.~Mohamed, A.~Agarwal, D.~Belgrave, K.~Cho, and A.~Oh (eds.), \emph{Advances in Neural Information Processing Systems}, volume~35, pp.\  36599--36612. Curran Associates, Inc., 2022.
\newblock \doi{10.52202/068431-2652}.
\newblock URL \url{https://proceedings.neurips.cc/paper_files/paper/2022/file/ed3cd2520148b577039adfade82a5566-Paper-Conference.pdf}.

\bibitem[Niu et~al.(2025)Niu, Ji, Liu, Zhao, Zhu, Zheng, Huang, Zhou, Hu, and Zhan]{niu2025h2o+}
Haoyi Niu, Tianying Ji, Bingqi Liu, Haocheng Zhao, Xiangyu Zhu, Jianying Zheng, Pengfei Huang, Guyue Zhou, Jianming Hu, and Xianyuan Zhan.
\newblock H2o+: an improved framework for hybrid offline-and-online rl with dynamics gaps.
\newblock In \emph{2025 IEEE International Conference on Robotics and Automation (ICRA)}, pp.\  1421--1428. IEEE, 2025.

\bibitem[Osband et~al.(2016)Osband, Blundell, Pritzel, and Van~Roy]{osband2016deep}
Ian Osband, Charles Blundell, Alexander Pritzel, and Benjamin Van~Roy.
\newblock Deep exploration via bootstrapped dqn.
\newblock \emph{Advances in neural information processing systems}, 29, 2016.

\bibitem[Ouyang et~al.(2022)Ouyang, Wu, Jiang, Almeida, Wainwright, Mishkin, Zhang, Agarwal, Slama, Ray, et~al.]{ouyang2022training}
Long Ouyang, Jeffrey Wu, Xu~Jiang, Diogo Almeida, Carroll Wainwright, Pamela Mishkin, Chong Zhang, Sandhini Agarwal, Katarina Slama, Alex Ray, et~al.
\newblock Training language models to follow instructions with human feedback.
\newblock \emph{Advances in neural information processing systems}, 35:\penalty0 27730--27744, 2022.

\bibitem[Peherstorfer et~al.(2016)Peherstorfer, Willcox, and Gunzburger]{peherstorfer2016optimal}
Benjamin Peherstorfer, Karen Willcox, and Max Gunzburger.
\newblock Optimal model management for multifidelity monte carlo estimation.
\newblock \emph{SIAM Journal on Scientific Computing}, 38\penalty0 (5):\penalty0 A3163--A3194, 2016.

\bibitem[Peng et~al.(2018)Peng, Andrychowicz, Zaremba, and Abbeel]{peng2018sim}
Xue~Bin Peng, Marcin Andrychowicz, Wojciech Zaremba, and Pieter Abbeel.
\newblock Sim-to-real transfer of robotic control with dynamics randomization.
\newblock In \emph{2018 IEEE international conference on robotics and automation (ICRA)}, pp.\  3803--3810. IEEE, 2018.

\bibitem[Ramos et~al.(2019)Ramos, Possas, and Fox]{ramos2019bayessim}
Fabio Ramos, Rafael~Carvalhaes Possas, and Dieter Fox.
\newblock Bayessim: adaptive domain randomization via probabilistic inference for robotics simulators.
\newblock \emph{arXiv preprint arXiv:1906.01728}, 2019.

\bibitem[Rudin et~al.(2022)Rudin, Hoeller, Reist, and Hutter]{rudin2022learning}
Nikita Rudin, David Hoeller, Philipp Reist, and Marco Hutter.
\newblock Learning to walk in minutes using massively parallel deep reinforcement learning.
\newblock In \emph{Conference on robot learning}, pp.\  91--100. PMLR, 2022.

\bibitem[Ryou et~al.(2026)Ryou, Wang, and Karaman]{ryou2026multi}
Gilhyun Ryou, Geoffrey Wang, and Sertac Karaman.
\newblock Multi-fidelity reinforcement learning for time-optimal quadrotor re-planning.
\newblock \emph{The International Journal of Robotics Research}, 45\penalty0 (4):\penalty0 645--669, 2026.

\bibitem[Sabatini et~al.(2026)Sabatini, Li, and Hutter]{sabatini2026bridging}
Gianluca Sabatini, Chenhao Li, and Marco Hutter.
\newblock Bridging the gap: Enabling soft actor critic for high performance legged locomotion.
\newblock \emph{arXiv preprint arXiv:2605.24975}, 2026.

\bibitem[Schulman et~al.(2015)Schulman, Moritz, Levine, Jordan, and Abbeel]{schulman2015high}
John Schulman, Philipp Moritz, Sergey Levine, Michael Jordan, and Pieter Abbeel.
\newblock High-dimensional continuous control using generalized advantage estimation.
\newblock \emph{arXiv preprint arXiv:1506.02438}, 2015.

\bibitem[Schulman et~al.(2017)Schulman, Wolski, Dhariwal, Radford, and Klimov]{schulman2017proximal}
John Schulman, Filip Wolski, Prafulla Dhariwal, Alec Radford, and Oleg Klimov.
\newblock Proximal policy optimization algorithms.
\newblock \emph{arXiv preprint arXiv:1707.06347}, 2017.

\bibitem[Schwarke et~al.(2025)Schwarke, Mittal, Rudin, Hoeller, and Hutter]{schwarke2025rsl}
Clemens Schwarke, Mayank Mittal, Nikita Rudin, David Hoeller, and Marco Hutter.
\newblock Rsl-rl: A learning library for robotics research.
\newblock \emph{arXiv preprint arXiv:2509.10771}, 2025.

\bibitem[Sifaou \& Simeone(2025)Sifaou and Simeone]{sifaou2025multi}
Houssem Sifaou and Osvaldo Simeone.
\newblock Multi-fidelity hybrid reinforcement learning via information gain maximization.
\newblock \emph{arXiv preprint arXiv:2509.14848}, 2025.

\bibitem[Smith et~al.(2022)Smith, Kew, Peng, Ha, Tan, and Levine]{smith2022legged}
Laura Smith, J~Chase Kew, Xue~Bin Peng, Sehoon Ha, Jie Tan, and Sergey Levine.
\newblock Legged robots that keep on learning: Fine-tuning locomotion policies in the real world.
\newblock In \emph{2022 international conference on robotics and automation (ICRA)}, pp.\  1593--1599. IEEE, 2022.

\bibitem[Tao et~al.(2024)Tao, Xiang, Shukla, Qin, Hinrichsen, Yuan, Bao, Lin, Liu, Chan, et~al.]{tao2024maniskill3}
Stone Tao, Fanbo Xiang, Arth Shukla, Yuzhe Qin, Xander Hinrichsen, Xiaodi Yuan, Chen Bao, Xinsong Lin, Yulin Liu, Tse-kai Chan, et~al.
\newblock Maniskill3: Gpu parallelized robotics simulation and rendering for generalizable embodied ai.
\newblock \emph{arXiv preprint arXiv:2410.00425}, 2024.

\bibitem[Team(2026)]{worldlabs2026atlas}
World~Labs Team.
\newblock Atlas: A world model for spatial intelligence.
\newblock \emph{World Labs Blog}, 2026.
\newblock https://www.worldlabs.ai/blog/atlas.

\bibitem[Tobin et~al.(2017)Tobin, Fong, Ray, Schneider, Zaremba, and Abbeel]{tobin2017domain}
Josh Tobin, Rachel Fong, Alex Ray, Jonas Schneider, Wojciech Zaremba, and Pieter Abbeel.
\newblock Domain randomization for transferring deep neural networks from simulation to the real world.
\newblock In \emph{2017 IEEE/RSJ international conference on intelligent robots and systems (IROS)}, pp.\  23--30. IEEE, 2017.

\bibitem[Wei et~al.(2025)Wei, Agarwal, Chen, Bosworth, Pfaff, and Tedrake]{wei2025empirical}
Adam Wei, Abhinav Agarwal, Boyuan Chen, Rohan Bosworth, Nicholas Pfaff, and Russ Tedrake.
\newblock Empirical analysis of sim-and-real cotraining of diffusion policies for planar pushing from pixels.
\newblock In \emph{2025 IEEE/RSJ International Conference on Intelligent Robots and Systems (IROS)}, pp.\  5625--5632. IEEE, 2025.

\bibitem[Williams(1992)]{williams1992simple}
Ronald~J Williams.
\newblock Simple statistical gradient-following algorithms for connectionist reinforcement learning.
\newblock \emph{Machine learning}, 8:\penalty0 229--256, 1992.

\bibitem[Xu et~al.(2023)Xu, Bai, Ma, Wang, Zhao, Wang, Li, and Li]{xu2023cross}
Kang Xu, Chenjia Bai, Xiaoteng Ma, Dong Wang, Bin Zhao, Zhen Wang, Xuelong Li, and Wei Li.
\newblock Cross-domain policy adaptation via value-guided data filtering.
\newblock \emph{Advances in Neural Information Processing Systems}, 36:\penalty0 73395--73421, 2023.

\end{thebibliography}
\bibliographystyle{iclr2027_conference}

\clearpage
\appendix
\section{Appendix}
\crefalias{section}{appendix}
\crefalias{subsection}{appendix}
\crefalias{subsubsection}{appendix}

\subsection{Proof of Lemma~\ref{lem:fit-granularity-vi}}
\label{app:proof-fit-granularity}

\Cref{lem:fit-granularity-vi} holds under the following regularity conditions: (i) the per-step statistics have finite second moments; (ii) $\operatorname{Var}(B^\low) > 0$ (if $\operatorname{Var}(B^\low)$ vanishes, the variance in question does not depend on $c$); (iii) every step of every episode is coupled, so the pair set of \cref{eq:mfpg-ppo} contains exactly one pair per $(k,t)$ and $|\mathcal{P}| = KH$.

\begin{proof} 
By condition (iii), $|\mathcal{P}| = KH$.
Grouping the sum by episode,
\begin{equation*}
L^{\mathrm{MF}}(\policyparams)
= \frac{1}{KH}\sum_{k=1}^{K}\sum_{t=1}^{H}\Big[\RV^\high_{k,t}
+ c\big(\RV^\low_{k,t} - \hat{\mu}^\low\big)\Big]
= \frac{1}{KH}\sum_{k=1}^{K}\big(B^\high_k + c\,B^\low_k\big)
\;-\; c\,\hat{\mu}^\low .
\end{equation*}
Since $\hat{\mu}^\low$ contributes no variance, and since
the episodes are independent and identically distributed,
\begin{align*}
\operatorname{Var}\big(L^{\mathrm{MF}}\big)
&= \frac{1}{K^2H^2}\sum_{k=1}^{K}\operatorname{Var}\big(B^\high_k + c\,B^\low_k\big)\\
&= \frac{1}{KH^2}\Big[\operatorname{Var}(B^\high)
+ 2c\,\operatorname{Cov}(B^\high, B^\low)
+ c^2\operatorname{Var}(B^\low)\Big],
\end{align*}
a strictly convex quadratic in $c$ whose unique minimizer is
$c^* = -\operatorname{Cov}(B^\high, B^\low)/\operatorname{Var}(B^\low)$. 
Covariance is linear in each of its two arguments, so
\begin{align*}
\operatorname{Cov}(B^\high,B^\low)
&= \operatorname{Cov}\!\left(
    \sum_{t=1}^{H}\RV^\high_t,
    \sum_{t'=1}^{H}\RV^\low_{t'}\right)\\
&= \sum_{t=1}^{H}\operatorname{Cov}\!\left(
    \RV^\high_t,\sum_{t'=1}^{H}\RV^\low_{t'}\right)\\
&= \sum_{t=1}^{H}\sum_{t'=1}^{H}
    \operatorname{Cov}(\RV^\high_t,\RV^\low_{t'}).
\end{align*}
Likewise, $\operatorname{Var}(B^\low) = \operatorname{Cov}(B^\low, B^\low) = \sum_{t=1}^{H}\sum_{t'=1}^{H}\operatorname{Cov}(\RV^\low_t, \RV^\low_{t'})$.

The per-step fit pools the $KH$ pairs and computes
moments centered at the means
$\bar{\RV}^\high := \frac{1}{KH}\sum_{k,t}\RV^\high_{k,t}$ and $\bar{\RV}^\low$:
\begin{equation*}
\widehat{\operatorname{Cov}}_{\mathrm{step}}
:= \frac{1}{KH}\sum_{k,t}\big(\RV^\high_{k,t}-\bar{\RV}^\high\big)
\big(\RV^\low_{k,t}-\bar{\RV}^\low\big),
\qquad
\widehat{\operatorname{Var}}_{\mathrm{step}}
:= \frac{1}{KH}\sum_{k,t}\big(\RV^\low_{k,t}-\bar{\RV}^\low\big)^2 .
\end{equation*}
The per-step fit is the ratio $\hat{c}_{\mathrm{step}} := -\widehat{\operatorname{Cov}}_{\mathrm{step}}/\widehat{\operatorname{Var}}_{\mathrm{step}}$.
By the strong law of large numbers over the independent episodes, as $K \to \infty$, almost surely
\begin{equation*}
\widehat{\operatorname{Cov}}_{\mathrm{step}}
\;\longrightarrow\;
\frac{1}{H}\sum_{t=1}^{H}\Big[\operatorname{Cov}\big(\RV^\high_t, \RV^\low_t\big)
+ \big(\mu^\high_t - \bar{\mu}^\high\big)\big(\mu^\low_t - \bar{\mu}^\low\big)\Big],
\end{equation*}
where $\mu_t := \E[\RV_t]$ and $\bar{\mu} := \frac{1}{H}\sum_{t}\mu_t$, and
analogously for $\widehat{\operatorname{Var}}_{\mathrm{step}}$. Under the
assumption that the per-step means are constant in $t$
\begin{equation*}
-\frac{\widehat{\operatorname{Cov}}_{\mathrm{step}}}
{\widehat{\operatorname{Var}}_{\mathrm{step}}}
\;\longrightarrow\;
-\frac{\sum_{t}\operatorname{Cov}(\RV^\high_t, \RV^\low_t)}
{\sum_{t}\operatorname{Var}(\RV^\low_t)}
= c_{\mathrm{step}} .
\end{equation*}

Splitting each double sum above into its same-step terms with $t = t'$ and its cross-step terms with $t \neq t'$, and using $\operatorname{Cov}(\RV^\low_t, \RV^\low_t) = \operatorname{Var}(\RV^\low_t)$, gives
\begin{equation*}
\operatorname{Cov}(B^\high, B^\low)
= \sum_{t}\operatorname{Cov}(\RV^\high_t, \RV^\low_t)
+ \sum_{t\neq t'}\operatorname{Cov}(\RV^\high_t, \RV^\low_{t'}),
\end{equation*}
\begin{equation*}
\operatorname{Var}(B^\low)
= \sum_{t}\operatorname{Var}(\RV^\low_t)
+ \sum_{t\neq t'}\operatorname{Cov}(\RV^\low_t, \RV^\low_{t'}).
\end{equation*}
If all cross-step covariances vanish, both off-diagonal sums are zero and
$c^* = c_{\mathrm{step}}$. They do not coincide in general: with $H = 2$,
zero-mean statistics, and $\RV^\low_1 := \RV^\high_2$,
$\RV^\low_2 := \RV^\high_1$ for independent $\RV^\high_1, \RV^\high_2$ with
positive variance, every
same-step covariance vanishes, so $c_{\mathrm{step}} = 0$, while
$B^\low = B^\high$ gives $c^* = -1$.
\end{proof}

\textbf{On the constant-means assumption.} Without the assumption that the per-step means $\E[\RV^\high_t]$ and $\E[\RV^\low_t]$ are constant in $t$, the limits
retain the terms
$\frac{1}{H}\sum_{t}(\mu^\high_t - \bar{\mu}^\high)(\mu^\low_t - \bar{\mu}^\low)$
in the numerator and $\frac{1}{H}\sum_{t}(\mu^\low_t - \bar{\mu}^\low)^2$ in the
denominator. The per-step fit then mixes the dependence structure with the
variation of the mean over the horizon. The assumption isolates the former, which is the effect of interest.

\subsection{Why Are the Uncoupled \ac{lf} Rollouts Not Re-synchronized?}
\label{sec:free-baseline}

The coupled \ac{lf} rollouts are re-synchronized to recorded \ac{hf} states to preserve correlation (\cref{sec:mfpg-ppo}), whereas the uncoupled \ac{lf} rollouts (which supply the \ac{lf} baseline $\hat{\mu}^\low$ in \cref{eq:mfpg-ppo}) evolve freely. Re-synchronizing the uncoupled \ac{lf} rollouts to the same states would remove the bias noted in \cref{sec:mfpg-ppo}, since the coupled and uncoupled \ac{lf} rollouts would then follow the same distribution. However, we show that a re-synchronized \ac{lf} baseline leaves a variance component that no fixed coefficient $c$ can reduce, no matter how faithful the \ac{lf} simulator is.

Fix the policy and let $C$ collect the recorded \ac{hf} states and the matched task variables for re-synchronization, and let $(\hfc{\RV^\high}, \lfc{\RV^\low})$ be the paired \ac{hf} and coupled \ac{lf} statistics, with finite positive variances and correlation $\rho$. By the law of total variance, $\Var(\RV^\high) = \Var(\E[\RV^\high \given C]) + \E[\Var(\RV^\high \given C)]$. We call the two terms the \emph{between-state} and the \emph{within-state} variance, with respect to the re-synchronization states in $C$. With freely evolving uncoupled \ac{lf} rollouts, the \ac{lf} baseline is a constant $\mu^\low_{\mathrm{u}}$, the mean of the uncoupled \ac{lf} statistic. If the uncoupled \ac{lf} rollouts are instead re-synchronized to $C$, the \ac{lf} baseline becomes the conditional mean $\mu^\low_C := \E[\RV^\low \given C]$. For a coefficient $c$, the control-variate estimators are
\begin{equation*}
Z_{\mathrm{u}} := \RV^\high + c\,\big(\RV^\low - \mu^\low_{\mathrm{u}}\big),
\qquad
Z_C := \RV^\high + c\,\big(\RV^\low - \mu^\low_C\big).
\end{equation*}
We treat both \ac{lf} baselines as exact. Finitely many uncoupled \ac{lf} rollouts add estimation noise to either.

\begin{lemma}[A re-synchronized \ac{lf} baseline cannot reduce the between-state variance]\label{lem:sync-baseline}
For every $c$, $\E[Z_C] = \E[\RV^\high]$ and
\begin{equation*}
\Var(Z_C) = \Var\big(\E[\RV^\high \given C]\big) + \E\big[\Var(\RV^\high + c\,\RV^\low \given C)\big]
\;\ge\; \Var\big(\E[\RV^\high \given C]\big),
\end{equation*}
whereas $\E[Z_{\mathrm{u}}] = \E[\RV^\high] + c\,\big(\E[\RV^\low] - \mu^\low_{\mathrm{u}}\big)$ and $\min_c \Var(Z_{\mathrm{u}}) = (1-\rho^2)\Var(\RV^\high)$.
\end{lemma}
\begin{proof}
Since $\mu^\low_C$ is a function of $C$, $\E[Z_C \given C] = \E[\RV^\high \given C]$ and $\Var(Z_C \given C) = \Var(\RV^\high + c\,\RV^\low \given C)$. The laws of total expectation and variance give the claims for $Z_C$. The \ac{lf} baseline of $Z_{\mathrm{u}}$ is a constant, so $\Var(Z_{\mathrm{u}}) = \Var(\RV^\high + c\,\RV^\low)$, a quadratic in $c$ minimized at ${c^* = -\Cov(\RV^\high, \RV^\low)/\Var(\RV^\low)}$ with the stated value.
\end{proof}

The between-state variance reflects which states the scarce \ac{hf} rollouts start from. A re-synchronized \ac{lf} baseline removes the between-state part of $\RV^\low$ from the \ac{lf} corrector, so the \ac{lf} correction acts only on the within-state variance and leaves $\E[\RV^\high \given C]$ unchanged, whereas the free \ac{lf} baseline keeps both parts, so that one coefficient can reduce both. The idealized case of perfect \ac{lf} statistics makes this concrete. If $\RV^\low = \RV^\high$ and $\mu^\low_{\mathrm{u}} = \E[\RV^\high]$, then $c = -1$ gives $Z_{\mathrm{u}} = \E[\RV^\high]$ with zero variance, whereas $Z_C = \E[\RV^\high \given C]$ retains the full between-state variance.

In our design, we leave the uncoupled \ac{lf} rollouts free to keep both variance components within reach of the control-variate estimator, at the price of a possible bias that vanishes when the coupled and uncoupled \ac{lf} statistics have the same mean. Reducing both variance components without bias would require an \ac{lf} baseline estimated from additional re-synchronization states. Such states come only from the \ac{hf} environment, precisely the data that are scarce.

\subsection{Budget Allocation}
\label{app:budget-allocation}

\paragraph{Sampling model and estimator.}
Fix the policy $\policyparams$ and consider an idealized sampling model. Each sample is a $T$-step rollout, represented by its mean surrogate $\RV_i^f := T^{-1}\sum_{t=1}^{T}\RV^f_{i,t}$. We draw $m^\high$ i.i.d.\ coupled pairs $(\hfc{\RV^\high_i}, \lfc{\RV^\low_i})$ with Pearson correlation $\rho$, $\rho^2 < 1$, and finite positive variances, together with $m^\low - m^\high$ additional independent \ac{lf} samples drawn from the same marginal distribution as the coupled \ac{lf} samples. Let $\hat{\mu}^\low_{\mathrm{all}} := \frac{1}{m^\low}\sum_{j=1}^{m^\low} \RV^\low_{j}$ be the baseline mean over all \ac{lf} samples. The estimator in \cref{eq:budget-alloc} is
\begin{equation}
L^{\mathrm{MF}}_{\mathrm{all}}(\policyparams) := \frac{1}{m^\high}\sum_{i=1}^{m^\high}\Big[\hfc{\RV^\high_i} + c\,\big(\lfc{\RV^\low_i} - \lfc{\hat{\mu}^\low_{\mathrm{all}}}\big)\Big].
\end{equation}

\paragraph{Variance and coefficient.}
For fixed $c$, accounting for the shared samples in the coupled \ac{lf} mean and the baseline gives
\begin{equation}
\begin{aligned}
\Var\big[L^{\mathrm{MF}}_{\mathrm{all}}\big]
&= \frac{\Var(\RV^\high)}{m^\high}
+ \left(\frac{1}{m^\high}-\frac{1}{m^\low}\right)
\Big[2c\,\Cov(\RV^\high,\RV^\low)+c^2\Var(\RV^\low)\Big].
\end{aligned}
\end{equation}
For $m^\low>m^\high$, minimizing this quadratic over $c$ at fixed counts gives
\begin{equation}
c^{*} = -\frac{\Cov(\RV^\high, \RV^\low)}{\Var(\RV^\low)},
\end{equation}
independent of $(m^\high, m^\low)$. Substituting $c^{*}$,
\begin{equation}\label{eq:budget-var-decomp}
\Var\big[L^{\mathrm{MF}}_{\mathrm{all}}\big] = \Var(\RV^\high)\left[\frac{1-\rho^2}{m^\high} + \frac{\rho^2}{m^\low}\right].
\end{equation}

\paragraph{Constrained allocation.}
For positive costs and budget, set $\nu=m^\low/m^\high\ge1$ and substitute $m^\high=p/(\samplecosth+\samplecostl\nu)$ into \cref{eq:budget-var-decomp} to obtain
\begin{equation}
\begin{aligned}
F(\nu)
&:= \frac{\Var(\RV^\high)}{p}(\samplecosth+\samplecostl\nu)
\left(1-\rho^2+\frac{\rho^2}{\nu}\right),\\
F'(\nu)
&= \frac{\Var(\RV^\high)}{p}
\left[\samplecostl(1-\rho^2)-\frac{\samplecosth\rho^2}{\nu^2}\right].
\end{aligned}
\end{equation}
Since $F'$ is nondecreasing, the minimizer on $\nu\ge1$ is $\nu^*=\max\{1,\tilde\nu\}$, where $\tilde{\nu} := \sqrt{\samplecosth\rho^2/(\samplecostl(1-\rho^2))}$, recovering the closed-form solution of \cref{sec:budget-allocation}. For $\nu^*>1$, this agrees with \citet[Thm.~3.4]{peherstorfer2016optimal}. At $\nu^*=1$, $m^\high=m^\low=p/(\samplecosth+\samplecostl)$, and the \ac{lf} correction vanishes for any $c$. Optimality is restricted to layouts retaining the paired statistics ($\nu \geq 1$), which keeps $\rho$ observable even when the control variate is disabled. Having $\nu^*>1$ alone does not guarantee lower variance than an equal-cost \ac{hf}-only allocation \citep[Corollary~3.5]{peherstorfer2016optimal}.

\paragraph{Integer counts and implementation.}
\Cref{eq:budget-alloc} is solved over continuous counts and the solution is floored to integers. Flooring preserves ordering and the budget cap but may leave budget unused or give $m^\high=0$. We clamp $m^\high\ge1$ and repair the counts to ensure $m^\low\ge m^\high\ge1$ and $\samplecosth m^\high+\samplecostl m^\low\le p$. Infeasible budgets are rejected.
While $m^\low<m^\high$, we increment $m^\low$ if the budget permits, otherwise decrement $m^\high$. While over budget, we decrement $m^\low$ if $m^\low>m^\high$, otherwise $m^\high$, always retaining $m^\high\ge1$.
Unused budget is not redistributed.
The resulting counts map onto the inputs of \cref{alg:mfpg} as $N^\high = m^\high$ and $N^\low = m^\low - m^\high$, with the coupled \ac{lf} samples included in $m^\low$.

\paragraph{Scope of the result.}
In practice, coupled \ac{lf} rollouts start from \ac{hf} states while uncoupled rollouts evolve freely, so their statistics need not share a marginal distribution. We also use per-transition correlation estimates, which omit cross-step covariance relevant to rollout-level allocation (\cref{lem:fit-granularity-vi}). The moments are estimated and smoothed online as the policy changes. The optimality statement therefore applies to the idealized scalar-estimator variance above, not to the deployed algorithm.

\subsection{Complete algorithm}
\label{sec:complete-algo}

\Cref{alg:mfpg} summarizes \ac{mfpg}-\ac{ppo}, with each iteration following the four steps of \cref{sec:mfpg-ppo}. Lines~\ref{line:collect} to~\ref{line:pairs} collect the multi-fidelity data, lines~\ref{line:adv} and~\ref{line:stat} construct the paired statistics, lines~\ref{line:fit} and~\ref{line:refit} fit the coefficient~$c$, and lines~\ref{line:fit} and~\ref{line:monitor} monitor the control variate. The policy is reparameterized as in \cref{sec:mfpg-ppo}. For a Gaussian policy with mean action $\bar{\action}_\policyparams(\state)$ and standard deviation $\sigma_\policyparams$, $g_\policyparams(\state, \varepsilon) = \bar{\action}_\policyparams(\state) + \sigma_\policyparams \odot \varepsilon$ with $\varepsilon \sim \mathcal{N}(0, I)$. \Cref{app:gating} specifies the monitoring procedures, and \cref{app:implementation-details} the implementation details. For each minibatch $\mathcal B$, $\mathcal B^\high$ indexes all
its \ac{hf} transitions, including those without a valid coupled
\ac{lf} counterpart.
$L^{\mathrm{MF}}(\policyparams)\big|_{\mathcal B}$ denotes
\cref{eq:mfpg-ppo} evaluated on that minibatch, with the
\ac{hf} term averaged only over transitions with valid
coupled \ac{lf} counterparts.

\begin{algorithm}[htbp]
    \caption{\ac{mfpg}-\ac{ppo}}
    \label{alg:mfpg}
    \begin{algorithmic}[1]
    \Require Rollout counts $N^\high, N^\low$; rollout length $T$; \ac{ppo} hyperparameters; iteration count~$\ell_{\max}$ 
    \State Initialize policy $\policy_\policyparams$, critics $V^\high, V^\low_{\mathrm{c}}, V^\low_{\mathrm{u}}$, and the monitoring state, cf.~\cref{app:gating}
    \For{iteration $\in \{1,2,\ldots,\ell_{\max}\}$}
      \State \textbf{Collect} $N^\high$ \ac{hf} rollouts with actions $\action^\high_t = g_\policyparams(\state^\high_t, \varepsilon_t)$, recording the start states, task variables, and noise $\varepsilon_t$ \label{line:collect}
      \State \textbf{Reset} each coupled \ac{lf} environment to its counterpart's recorded start state and task variables, then roll out $T$ steps with actions $\action^\low_t = g_\policyparams(\state^\low_t, \varepsilon_t)$ using the recorded noise $\varepsilon_t$, cf.~\cref{fig:sampling-pipeline}
      \State \textbf{Collect} $N^\low$ uncoupled $T$-step \ac{lf} rollouts with independently drawn noise and free resets
      \State \textbf{Form} the index sets $\mathcal{P}$ of coupled pairs and $\mathcal{U}$ of uncoupled \ac{lf} transitions \label{line:pairs}
      \State \textbf{Estimate advantages} for each batch with the corresponding critic over the \ac{gae} horizons of \cref{fig:sampling-pipeline} \label{line:adv}
      \State \textbf{Evaluate} the statistics $\RV^\high_{k,t}(\policyparams), \RV^\low_{k,t}(\policyparams), \hat{\mu}^\low(\policyparams)$ of \cref{eq:mfpg-ppo} for all three batches \label{line:stat}
      \State \textbf{Fit and smooth} $c$ on the pairs in $\mathcal{P}$, cf.~\cref{sec:fitting}; run the iteration-level check, cf.~\cref{app:gating} \label{line:fit}
      \For{each epoch and minibatch $\mathcal B$}
        \State \textbf{Re-evaluate} the statistics $\RV^\high_{k,t}(\policyparams), \RV^\low_{k,t}(\policyparams), \hat{\mu}^\low(\policyparams)$ of \cref{eq:mfpg-ppo} on $\mathcal B$ under the current policy; re-fit $c$ on the pairs in $\mathcal B$ \label{line:refit}
        \State \textbf{Monitor} with the minibatch checks of \cref{app:gating}; smooth $c$ if the evidence check passes, apply the control variate only if all checks pass, and set $c \gets 0$ otherwise \label{line:monitor}
        \State \textbf{Set} the policy objective
        $L_{\mathcal B}(\policyparams) \gets
        \begin{cases}
            L^{\mathrm{MF}}(\policyparams)\big|_{\mathcal B}, & c \ne 0, \\[3pt]
            \displaystyle\frac{1}{|\mathcal B^\high|}
            \sum_{(k,t)\in\mathcal B^\high}
            \RV^\high_{k,t}(\policyparams), & c = 0.
        \end{cases}$
        \State \textbf{Update} with a \ac{ppo} step on $L_{\mathcal B}(\policyparams)$; fit each critic on its corresponding transitions with targets that never span a re-synchronized state
      \EndFor
    \EndFor
    \end{algorithmic}
\end{algorithm}

\textbf{Budget-allocation variant.}
Before each iteration's data collection, we set $N^\high$
and $N^\low$ using the allocation and integer-rounding
procedure in \cref{app:budget-allocation}.
The objective $L^{\mathrm{MF}}_{\mathrm{all}}$ replaces
the baseline mean in $L^{\mathrm{MF}}$ with the mean over
all \ac{lf} transitions in each minibatch.
\subsection{Monitoring the Control Variate}\label{app:gating}

Throughout, $\hat c$ denotes the coefficient fitted on the current fitting data, $\bar c$ the running smoothed coefficient, and $c$ the coefficient applied in \cref{eq:mfpg-ppo}: $c=\bar c$ when all applicable checks pass and $c=0$ otherwise.

\textbf{Simulation.} We fit and monitor the control variate
at both batch and minibatch levels. At iteration start, we fit
$\hat c=-\widehat{\operatorname{Cov}}(\RV^\high,\RV^\low)/
\widehat{\operatorname{Var}}(\RV^\low)$
on all paired transitions in the batch, update
$\bar c\leftarrow(1-\alpha)\bar c+\alpha\hat c$,
and check variance reduction. If the check fails, we disable the control variate for the iteration.
For each minibatch of an eligible iteration, we re-fit
$\hat c$ from its paired transitions and obtain a bootstrap confidence
interval~\citep{efron1992bootstrap} by resampling those pairs. If the interval contains zero,
we disable the control variate for that minibatch. Otherwise, we update
$\bar c$ and apply the control variate using $\bar c$ only if the variance-reduction
check passes.

\textbf{Real robot.} With only 4 paired episodes per iteration, within-iteration bootstrap intervals can be unstable. Thus, we assess statistical evidence once per iteration, using a Fisher-$z$ heuristic to accumulate correlation evidence
across iterations. The transform approximately normalizes the sampling distribution
of correlations and stabilizes its variance. 
At iteration start, we average the Fisher-$z$~\citep{fisher1921probable} transforms
of the per-iteration correlations:
$\bar z = |\mathcal W|^{-1}\sum_{m\in\mathcal W}
\operatorname{atanh}\!\big(
\widehat{\operatorname{Corr}}_k(B^\high_{m,k},B^\low_{m,k})\big)$,
where $m$ indexes iterations, $k$ indexes episodes, and
$\mathcal W$ contains up to $M$ most recent iterations. We form the
approximate $(1-\delta)$ confidence interval
$\bar z \pm z_{1-\delta/2}/\sqrt{|\mathcal W|}$,
where $z_{1-\delta/2}$ is the standard-normal quantile.
If the interval contains zero, we disable the control variate
for the iteration. We update the control-variate coefficient at iteration start
and, if the interval excludes zero, before each minibatch update:
$\bar c=-\operatorname{EMA}_{\alpha}
[\widehat{\operatorname{Cov}}_k(B^\high_k,B^\low_k)]/
\operatorname{EMA}_{\alpha}
[\widehat{\operatorname{Var}}_k(B^\low_k)]$ under the current policy on \emph{all} paired episode sums from this iteration. Then, each minibatch applies the correction only if the variance check passes.

\textbf{Variance check and fallback.} We compare the sample variance of $\hfc{\RV^\high}+\bar c(\lfc{\RV^\low}-\lfc{\hat{\mu}^\low_{\mathcal B}})$ with that of $\hfc{\RV^\high}$, using individual paired steps in simulation and sums over paired episodes on the robot. A variance ratio greater than or equal to 1 sets $c=0$, yielding an \ac{hf}-only actor objective.

\begin{table}[htbp]
  \centering
  \caption{Control variate monitoring hyperparameters.}
  \label{tab:monitoring}
  \small
  \begin{tabular}{@{}lcc@{}}
      \toprule
      Hyperparameter & Simulation & Real robot \\
      \midrule
      EMA rate $\alpha$ & $0.002$ (FR); $0.025$ (BA) & $0.05$ \\
      Bootstrap resamples & $10^3$ & --- \\
      Evidence-window length $M$ & --- & $4$ \\
      Confidence level $1-\delta$ & $0.9$ & $0.9$ \\
      \bottomrule
  \end{tabular}
\end{table}

\subsection{Implementation and experiment details}
\label{app:implementation-details}

\subsubsection{Hyperparameter tuning protocol}\label{sec:hyperparam}

\textbf{Simulation tuning protocol.} All methods share the same tuning protocol. We tune on H1 Rough and ANYmal-D Rough with $N^\high=256$ and 3 seeds per task. Selected hyperparameters remain fixed across all locomotion tasks and \ac{hf} data budgets. \ac{mfpg}-\ac{ppo}-FR adds $N^\low=4096$ uncoupled \ac{lf} environments, and the other \ac{lf}-using baselines match its total \ac{lf} count. 

We select hyperparameters using both final and cumulative \ac{hf} training performance, measured by final return and \ac{auc}~\citep{osband2016deep,hessel2018rainbow}. Within each method, let \(M_{d,k}(q)\) denote metric \(k\in\{\mathrm{final},\mathrm{AUC}\}\), averaged over seeds, for configuration \(q\) on task \(d\). We standardize each metric across candidate hyperparameter configurations as \(z_{d,k}(q)=[M_{d,k}(q)-\mu_{d,k}]/s_{d,k}\), where \(\mu_{d,k}\) and \(s_{d,k}\) are the mean and sample standard deviation across candidates. We select the configuration maximizing \(S(q)=\sum_d\sum_k z_{d,k}(q)\).

We tune the \ac{ppo} backbone independently for
\ac{ppo} \ac{hf}-Only, \ac{mfpg}-\ac{ppo}-FR,
Co-training, DARC, and PAR, then tune method-specific
hyperparameters with the selected backbone fixed. For each \ac{ppo} backbone we tune, we first search learning rate/minibatch pairs at the default of 5 epochs, then search the number of epochs for the 4 best pairs. The \ac{ppo} grids are centered on the values of Isaac Lab's task-default RSL-RL configurations~\citep{mittal2025isaaclab,schwarke2025rsl}.
\ac{mfpg}-\ac{ppo}-BA inherits the selected
\ac{mfpg}-\ac{ppo}-FR backbone and tunes only its
\ac{ema} rate. 
For \ac{mfpg}-REINFORCE, we first search the learning rate on a grid centered on the default of the released implementation of \citet{liu2026a}, at its default discount factor of $0.97$ and \ac{ema} rate of $0.05$, then jointly search both for the selected learning rate.
For \ac{sac} \ac{hf}-Only, we first search pairs of learning rate and \ac{utd} ratio at the default target-entropy scale of $0.167$, then search this scale for the 4 best pairs. The number of minibatch updates per iteration sets the \ac{utd} ratio, and the target entropy equals the negative scale times the action dimension. \Citet{sabatini2026bridging} use $200$ minibatches per iteration with $8192$ parallel environments, which implies a \ac{utd} ratio of about $8.3$. We test \ac{utd} ratios of $\{8,16,32\}$ and this fixed $200$-minibatch setting.
\Cref{tab:simulation-hp-search} summarizes the grid values for the hyperparameter search.

\begin{table}[!htb]
    \centering
    \caption{Grid values for simulation hyperparameter search.}
    \label{tab:simulation-hp-search}
    \small
    \setlength{\tabcolsep}{4pt}
    \begin{tabular}{@{}
        >{\raggedright\arraybackslash}p{0.28\linewidth}
        >{\raggedright\arraybackslash}p{0.69\linewidth}@{}}
        \toprule
        Method & Candidate values \\
        \midrule
        \multicolumn{2}{@{}l}{\emph{Backbone parameters}} \\
        \addlinespace
        \ac{ppo}
        & Learning rate $\{5\times10^{-4},10^{-3},2\times10^{-3}\}$;
          epochs $\{3,5,8\}$;
          minibatches per epoch $\{2,4,8\}$ \\
        \addlinespace
        REINFORCE
        & Learning rate $\{7\times10^{-5},2\times10^{-4},
          7\times10^{-4},2\times10^{-3},7\times10^{-3}\}$;
          discount factor $\{0.95,0.97,0.99\}$ \\
        \addlinespace
        \ac{sac}
        & Actor and critic learning rate $\{10^{-4},2\times10^{-4},5\times10^{-4}\}$;
          \ac{utd} ratio $\{8,16,32\}$ or $200$ minibatches per iteration;
          target-entropy scale $\{0.05,0.167,0.5\}$ \\
        \midrule
        \multicolumn{2}{@{}l}{\emph{Method-specific parameters}} \\
        \addlinespace
        \ac{mfpg}-\ac{ppo}-FR
        & \ac{ema} rate $\{0.0005,0.002,0.01\}$ \\
        \addlinespace
        \ac{mfpg}-\ac{ppo}-BA
        & \ac{ema} rate $\{0.025,0.05,0.1\}$ \\
        \addlinespace
        DARC / PAR
        & Reward-correction strength $\{0.005,0.1,0.5,1,2,5\}$ \\
        \addlinespace
        \ac{mfpg}-REINFORCE
        & \ac{ema} rate $\{0.025,0.05,0.1\}$ \\
        \bottomrule
    \end{tabular}
\end{table}

\textbf{Simulation hyperparameters.}
Untuned \ac{ppo} backbone hyperparameters follow Isaac Lab's
task-default RSL-RL configurations~\citep{mittal2025isaaclab,schwarke2025rsl} for all \ac{ppo}-based methods. \ac{ppo} \ac{lf}-Only uses the default backbone because it operates on the default $4096$ environments. \ac{mfpg}-REINFORCE uses the same actor and value network architectures as the \ac{ppo}-based methods, and its other untuned hyperparameters, e.g., a value-loss coefficient and maximum gradient norm of $1$, follow the released implementation of \citet{liu2026a}. Untuned \ac{sac} hyperparameters follow the task-specific released configurations of \citet{sabatini2026bridging}.
\Cref{tab:simulation-hyperparameters} lists the adopted hyperparameters.

\begin{table}[!htb]
    \centering
    \caption{Simulation hyperparameters.}
    \label{tab:simulation-hyperparameters}
    \small
    \setlength{\tabcolsep}{4pt}

    \begin{tabular}{@{}
        >{\raggedright\arraybackslash}p{0.43\linewidth}
        >{\raggedright\arraybackslash}p{0.54\linewidth}@{}}
        \toprule
        \multicolumn{2}{@{}l}{\emph{\ac{ppo} backbone hyperparameters inherited from Isaac Lab's task-default RSL-RL configurations}} \\
        \addlinespace
        Parameter & Value \\
        \midrule
        Rollout length per environment & 24 steps \\
        Optimizer & Adam \\
        Discount factor / \ac{gae} parameter & $0.99$ / $0.95$ \\
        \ac{ppo} clipping parameter & $0.2$ \\
        Value-loss coefficient / value clipping & $1$ / enabled \\
        Entropy coefficient & H1: $0.01$; ANYmal-D: $0.005$ \\
        Learning-rate schedule & KL-adaptive; target KL divergence $0.01$ \\
        Maximum gradient norm & $1$ \\
        Actor and critic hidden-layer sizes
        & Rough: $(512,256,128)$; flat: $(128,128,128)$ \\
        Activation function & ELU \\
        Observation normalization & Disabled \\
        Initial policy standard deviation & $1$ \\
        \bottomrule
    \end{tabular}

    \par\medskip

    \begin{tabular}{@{}
        >{\raggedright\arraybackslash}p{0.43\linewidth}
        >{\raggedright\arraybackslash}p{0.54\linewidth}@{}}
        \toprule
        \multicolumn{2}{@{}l}{\emph{\ac{sac} hyperparameters inherited from the released configurations of \citet{sabatini2026bridging}}} \\
        \addlinespace
        Parameter & Value \\
        \midrule
        Rollout length per environment & 24 steps \\
        Optimizer & Adam \\
        Discount factor & Rough: $0.97$; flat: $0.96$ \\
        Multi-step return horizon & $5$ steps \\
        Replay buffer size & $5\times10^{6}$ transitions \\
        Minibatch size & $8192$ transitions \\
        Temperature & Initial $0.001$; automatically tuned with learning rate $2\times10^{-5}$ \\
        Target smoothing coefficient & $0.003$ \\
        Maximum gradient norm & $1$ \\
        Actor and critic hidden-layer sizes & Rough: $(1024,512,256)$; flat: $(512,256,128)$ \\
        Activation function & SiLU \\
        Observation normalization & Enabled \\
        Initial policy standard deviation & $0.15$ \\
        \bottomrule
    \end{tabular}

    \par\medskip

    \begin{tabular}{@{}
        >{\raggedright\arraybackslash}p{0.23\linewidth}
        >{\centering\arraybackslash}p{0.14\linewidth}
        >{\centering\arraybackslash}p{0.08\linewidth}
        >{\centering\arraybackslash}p{0.13\linewidth}
        >{\raggedright\arraybackslash}p{0.33\linewidth}@{}}
        \toprule
        \multicolumn{5}{@{}l}{\emph{Method configurations}} \\
        \addlinespace
        Method & Learning rate & Epochs & Minibatches per epoch
        & Additional hyperparameters \\
        \midrule
        \ac{ppo} \ac{hf}-Only
        & $2\times10^{-3}$ & 5 & 8 & --- \\
        \addlinespace
        \ac{mfpg}-\ac{ppo}-FR
        & $5\times10^{-4}$ & 8 & 8
        & \ac{ema} rate $0.002$ \\
        \addlinespace
        \ac{mfpg}-\ac{ppo}-BA
        & $5\times10^{-4}$ & 8 & 8
        & \ac{ema} rate $0.025$ \\
        \addlinespace
        Co-training
        & $5\times10^{-4}$ & 8 & 8
        & --- \\
        \addlinespace
        DARC
        & $2\times10^{-3}$ & 3 & 8
        & Reward-correction strength $0.005$ \\
        \addlinespace
        PAR
        & $2\times10^{-3}$ & 8 & 8
        & Reward-correction strength $0.005$; equal \ac{hf}/\ac{lf} loss weights \\
        \addlinespace
        \ac{mfpg}-REINFORCE
        & $7\times10^{-5}$ & --- & ---
        & Discount factor $0.97$; \ac{ema} rate $0.025$ \\
        \addlinespace
        \ac{ppo} \ac{lf}-Only 
        & $10^{-3}$ & 5 & 4 & --- \\
        \addlinespace
        \ac{sac} \ac{hf}-Only
        & $5\times10^{-4}$ & 1 & $\frac{32\times24N^\high}{8192}$
        & \ac{utd} ratio $32$; target-entropy scale $0.5$ \\
        \bottomrule
    \end{tabular}
\end{table}

\textbf{Real-robot hyperparameters.} In real-robot experiments, we first tune the shared \ac{ppo} backbone (adopted from  ManiSkill's~\citep{tao2024maniskill3} PickCube \ac{ppo} configuration) over learning rates $\{1\times10^{-5}, 4\times10^{-5}, 8\times10^{-5}, 1.5\times10^{-4}, 3\times10^{-4}\}$ and minibatch counts per epoch $\{4, 6, 8, 32\}$. We use the selected backbone for both \ac{mfpg}-\ac{ppo} and \ac{ppo} \ac{hf}-Only. On top of it, we tune the \ac{mfpg}-specific \ac{ema} rate $\{0.05, 0.1\}$ and \ac{lf} rollout segment length $\{10, 20, 60\}$ steps. We additionally adopt a learning-rate schedule for both methods. The learning rate steps down from $4\times10^{-5}$ to $2.5\times10^{-5}$ and then to $1.5\times10^{-5}$ when the mean training success rate over the latest $10$ updates reaches $0.8$ and $0.925$. \Cref{tab:robot-hp-settings} lists the adopted hyperparameters.

\begin{table}[!htb]
    \centering
    \caption{Real-robot hyperparameters.}
    \label{tab:robot-hp-settings}
    \small
    \setlength{\tabcolsep}{4pt}
    \begin{tabular}{@{}
        >{\raggedright\arraybackslash}p{0.50\linewidth}
        >{\raggedright\arraybackslash}p{0.47\linewidth}@{}}
        \toprule
        Parameter & Value \\
        \midrule
        \multicolumn{2}{@{}l}{\emph{Shared}} \\
        \addlinespace
        Initial learning rate & $4\times10^{-5}$ \\
        Epochs & 4 \\
        Minibatches per epoch & 4 \\
        \ac{hf} episodes per update & 4 \\
        Optimizer & Adam \\
        Discount factor / \ac{gae} parameter & $0.8$ / $0.95$ \\
        \ac{ppo} clipping parameter & $0.2$ \\
        Value-loss coefficient / value clipping & $0.5$ / disabled \\
        Entropy coefficient & $0$ \\
        Maximum gradient norm & $0.5$ \\
        Actor and critic hidden-layer sizes & $(256,256)$ \\
        Activation function & ELU \\
        Observation normalization & Disabled \\
        \midrule
        \multicolumn{2}{@{}l}{\emph{\ac{mfpg}-\ac{ppo}-specific}} \\
        \addlinespace
        \ac{ema} rate & $0.05$ \\
        Coupled \ac{lf} rollout-segment length & 20 steps \\
        Uncoupled \ac{lf} environments & 8192 \\
        \bottomrule
    \end{tabular}
\end{table}

\subsubsection{Implementation details}
\label{app:method-implementation}

\textbf{Two-process implementation.} To support different simulator fidelities or a physical robot as the \ac{hf} environment, we implement \ac{mfpg}-\ac{ppo} with two processes. An \emph{\ac{hf} worker} collects \ac{hf} rollouts, and a \emph{learner} collects coupled and uncoupled \ac{lf} rollouts and optimizes the policy and critics. The learner sends the policy parameters and action noise to the worker, then uses the returned \ac{hf} states to initialize coupled \ac{lf} rollouts before collecting \ac{lf} data. For locomotion, both processes use Isaac Lab~\citep{mittal2025isaaclab}, and for Franka robot learning, the worker controls the robot and the learner runs ManiSkill~\citep{tao2024maniskill3}. In locomotion, coupled \ac{lf} environments match \ac{hf} velocity-tracking commands, and coupled and uncoupled \ac{lf} environments inherit \ac{hf} terrain assignments. In manipulation, coupled \ac{lf} environments match the measured \ac{hf} object pose.

\textbf{Co-training.} We adapt Co-training method~\citep{wei2025empirical} to on-policy \ac{rl} setting and train a shared policy and critic on independently collected \ac{hf} and \ac{lf} rollouts. Let $\mathcal L_{\mathrm{PPO}}^f$ be the mean loss in fidelity $f$, including policy, value, and entropy terms, and $n^f$ its transition count.
Co-training minimizes
\begin{equation*}
    \mathcal L_{\mathrm{co}}
    = \frac{n^\high\mathcal L_{\mathrm{PPO}}^\high
            +n^\low\mathcal L_{\mathrm{PPO}}^\low}
           {n^\high+n^\low},
\end{equation*}
giving each transition equal weight.

\textbf{DARC.}
Domain Adaptation with Rewards from Classifiers (DARC)~\citep{eysenbach2021offdynamics} corrects \ac{lf} rewards using a classifier-based estimate of the dynamics mismatch.
We adapt this correction to \ac{ppo}.
Two binary domain classifiers predict the fidelity label $f$: $D_{sa}(f\mid s,a)$ from state-action pairs and $D_{sas}(f\mid s,a,s')$ from transitions.
Together, they estimate the relative likelihood of observing $s'$ under \ac{hf} versus \ac{lf} dynamics, given $(s,a)$.
The resulting reward correction is
\begin{equation*}
    \widetilde r^\low
    = r^\low+\eta_{\mathrm{DARC}}
    \left[
        \log\frac{D_{sas}(\high\mid s,a,s')}{D_{sas}(\low\mid s,a,s')}
        -\log\frac{D_{sa}(\high\mid s,a)}{D_{sa}(\low\mid s,a)}
    \right],
\end{equation*}
where $\eta_{\mathrm{DARC}}$ controls the correction strength. 
We follow \citet{eysenbach2021offdynamics} and the authors'
released code\footnote{\url{https://github.com/google-research/google-research/tree/master/darc}},
using the classifier architecture and batch size specified in \citet[Appendix~E]{eysenbach2021offdynamics}. The policy and critic train only on \ac{lf} rollouts with corrected rewards. Both fidelities train the classifiers. We use two networks with hidden layers $(256,256)$ and ReLU activations, add the state-action network's logits to the transition network's logits, and apply Gaussian input noise with standard deviation $1$ only during classifier training.
At each classifier optimization step, we sample $128$
transitions from each fidelity's replay buffer, forming minibatches
$\mathcal B^\high$ and $\mathcal B^\low$, and minimize
\begin{equation*}
    \mathcal L_{\mathrm{cls}}
    = -\frac{1}{256}
    \sum_{f\in\{\high,\low\}}
    \sum_{(s,a,s')\in\mathcal B^f}
    \left[
        \log D_{sa}(f\mid s,a)
        + \log D_{sas}(f\mid s,a,s')
    \right].
\end{equation*}
Classifier training starts in the first iteration. In our implementation, reward correction begins after the first $10\%$ of training iterations.

\textbf{PAR.}
Policy Adaptation by Representation Mismatch (PAR)~\citep{lyu2024cross}
penalizes \ac{lf} transitions using the mismatch between the embedding
of the observed next state and the embedding predicted by a
representation model trained on \ac{hf} data.
We follow~\citet{lyu2024cross} and the authors' released
code\footnote{\url{https://github.com/dmksjfl/PAR}},
adapting the policy and critic updates to \ac{ppo}.
Let $f_{\bar\psi}$ and $g_{\bar\psi}$ be target copies of a state
encoder and a next-state embedding predictor, respectively,
with parameters $\bar\psi$ maintained by an \ac{ema} of the online parameters $\psi$.
For embedding dimension $d_z$ and correction strength
$\eta_{\mathrm{PAR}}$, the corrected \ac{lf} reward is
\begin{equation*}
    \widetilde r^\low
    = r^\low-\frac{\eta_{\mathrm{PAR}}}{d_z}
      \left\|g_{\bar\psi}\!\left(f_{\bar\psi}(s),a\right)
                   -f_{\bar\psi}(s')\right\|_2^2.
\end{equation*}
We train a shared policy and critic on \ac{hf} rollouts with
their original rewards and \ac{lf} rollouts with corrected rewards.
Let $\mathcal L_{\mathrm{PPO}}^f$ denote the mean \ac{ppo} loss
over transitions from fidelity $f$, including policy, value,
and entropy terms.
Following PAR's equal weighting of the two domains, we minimize
\begin{equation*}
    \mathcal L_{\mathrm{PAR}}
    = \tfrac12\mathcal L_{\mathrm{PPO}}^\high
      +\tfrac12\mathcal L_{\mathrm{PPO}}^\low.
\end{equation*}
The encoder and predictor each have
hidden layers $(256,256)$ with ELU activations, and $d_z=256$.
We normalize each state embedding by dividing it by the mean
absolute value of its components.
On \ac{hf} transitions $(s,a,s')$, we train the encoder and
predictor to predict the next-state embedding from the current
state and action, minimizing
$\|g_\psi(f_\psi(s),a)-\operatorname{sg}(f_\psi(s'))\|_2^2/d_z$,
where $\operatorname{sg}$ treats its argument as fixed during differentiation.
In our implementation, after completing all \ac{ppo} epochs, we perform the same number
of gradient updates to the encoder and predictor as to the policy.
Each update uses $128$ transitions sampled from the \ac{hf}
replay buffer.
After each gradient update, we update the target parameters as
$\bar\psi \gets 0.995\,\bar\psi + 0.005\,\psi$. Reward correction is active from the first iteration.

\textbf{Shared reward correction details.}
DARC and PAR correct rewards before \ac{gae} computation. 
Auxiliary networks for reward correction remain fixed during policy optimization, followed by as many auxiliary gradient steps as policy update steps.
Both DARC and PAR use Adam with learning rate $3\times10^{-4}$ for auxiliary networks and replay buffer size $10^6$ per stored fidelity: both fidelities for DARC, only \ac{hf} for PAR.
\Cref{tab:simulation-hyperparameters} gives the selected correction strengths.

\textbf{\ac{mfpg}-REINFORCE.} We follow \citet{liu2026a} and the authors' released implementation\footnote{\url{https://github.com/CLeARoboticsLab/MultiFidelityPolicyGradients}}. Each iteration, \ac{mfpg}-REINFORCE optimizes the objective with the episode-level statistic $\RV_\trajectory(\policyparams)$ defined in \cref{sec:preliminaries}
\begin{equation*}
    \frac{1}{N^\high}
    \sum_{k=1}^{N^\high}
    \Big[\RV_{\trajectory^\high_k}(\policyparams)
    +c\,\big(\RV_{\trajectory^\low_k}(\policyparams)
             -\hat{\mu}^\low_{\mathrm{all}}(\policyparams)\big)\Big].
\end{equation*}
Here, $k$ indexes coupled rollout pairs, and $\hat{\mu}^\low_{\mathrm{all}}$ averages the statistics over all coupled and uncoupled \ac{lf} rollouts. Each iteration, \ac{mfpg}-REINFORCE collects parallel $100$-step multi-fidelity rollouts.
Following the released implementation of \citet{liu2026a},
it approximates full-episode Monte Carlo returns using rewards up to
episode termination or the rollout boundary, without value bootstrapping. A shared value function $V^\high$, fitted to returns
on paired \ac{hf} transitions, supplies the baseline for both
fidelities. After each collection,
\ac{mfpg}-REINFORCE fits $c$ per step as described in
\cref{sec:fitting}, smoothing the correlation and standard
deviations separately with an \ac{ema}.
It sets $c=0$ when the current, unsmoothed
correlation is negative.
For the policy update, it differentiates all three terms, holding
the returns, value baselines, and $c$ fixed.
Each collection produces one full-batch policy and value update.

\textbf{\ac{sac} \ac{hf}-Only.} We use the \ac{sac} implementation and Isaac Lab configurations released by \citet{sabatini2026bridging}\footnote{\url{https://github.com/leggedrobotics/rsl_rl_sac} and \url{https://github.com/sabagian/isaaclab-sac}}. At the start of \ac{hf} training, we restore the actor, critics, and observation-normalization statistics from the task's \ac{lf}-pretrained checkpoint described in \cref{app:experimental-protocol}, and reinitialize the optimizers and temperature. The replay buffer starts empty and stores only \ac{hf} transitions. Each iteration collects a 24-step rollout from each of the $N^\high$ \ac{hf} environments. The first $8192/N^\high$ iterations collect a total of $24\times8192$ transitions without gradient updates, matching the released configuration's single collection-only iteration with $8192$ environments. These iterations count toward the \ac{hf} sampling budget. Subsequent iterations perform the minibatch updates in \cref{tab:simulation-hyperparameters}.

\textbf{Sampling budgets.}
\Cref{tab:simulation-sampling-counts} summarizes per-iteration sampling budgets for all the methods. For \ac{mfpg}-\ac{ppo}-BA, we assume costs $\samplecosth=10$ and $\samplecostl=1$ per rollout and set $p=10N^\high_{\mathrm{ref}}+4096$, where $N^\high_{\mathrm{ref}}$ is the nominal \ac{hf} environment count of the \ac{mfpg}-\ac{ppo}-FR with the corresponding total cost. 

\begin{table}[!htb]
    \centering
    \caption{Simulation rollouts per iteration. Rollouts contain 24 steps, except for \ac{mfpg}-REINFORCE (100 steps).}
    \label{tab:simulation-sampling-counts}
    \small
    \setlength{\tabcolsep}{5pt}
    \begin{tabular}{@{}lccc@{}}
        \toprule
        Method & \ac{hf} & Coupled \ac{lf} & Uncoupled \ac{lf} \\
        \midrule
        \ac{ppo} \ac{hf}-Only & $N^\high$ & --- & --- \\
        \ac{ppo} \ac{lf}-Only & --- & --- & $4096$ \\
        \ac{mfpg}-\ac{ppo}-FR & $N^\high$ & $N^\high$ & $4096$ \\
        \ac{mfpg}-\ac{ppo}-BA & $m^\high$ & $m^\high$ & $m^\low-m^\high$ \\
        Co-training, DARC, PAR & $N^\high$ & --- & $4096+N^\high$ \\
        \ac{mfpg}-REINFORCE & $N^\high$ & $N^\high$ & $4096$ \\
        \ac{sac} \ac{hf}-Only & $N^\high$ & --- & --- \\
        \bottomrule
    \end{tabular}
\end{table}

\subsubsection{Experiment details}
\label{app:experimental-protocol}

\textbf{Simulation locomotion tasks and fidelities.}
We use Isaac Lab's Unitree H1 and ANYmal-D velocity-tracking tasks on flat and rough terrain\footnote{The base task IDs are \path{Isaac-Velocity-Flat-H1-v0}, \path{Isaac-Velocity-Rough-H1-v0}, \path{Isaac-Velocity-Flat-Anymal-D-v0}, and \path{Isaac-Velocity-Rough-Anymal-D-v0}.}, with task-default observations, rewards, termination rules, and action space.
Actions specify joint-position targets.
Observations comprise base linear and angular velocities, projected gravity, velocity commands, joint positions relative to the default pose and joint velocities, and the previous action. Rough-terrain tasks additionally include a height scan.
We disable injected observation noise.
Both fidelities act every $0.02$\,s, with the physics settings in \cref{tab:locomotion-fidelity-details}.
The \ac{lf} assets use simplified collision geometry visualized in \cref{fig:locomotion-collision-models}, and ANYmal-D \ac{lf} additionally approximates cylinder and cone collisions.
The cross-task comparison uses $N^\high\in\{64,128,256,512\}$. The H1 Rough scaling study additionally includes $N^\high\in\{8,16,32\}$.

\begin{figure}[!htb]
    \centering
    \begin{minipage}{0.85\linewidth}
        \begin{minipage}[t]{0.47\linewidth}
            \centering
            {\small\textbf{(a) Unitree H1}}\par\smallskip
            {\small\makebox[0.48\linewidth][c]{Full (\acs{hf})}\hfill\makebox[0.48\linewidth][c]{Simplified (\acs{lf})}}\par\smallskip
            \includegraphics[width=\linewidth,trim=0bp 8bp 398bp 189bp,clip]{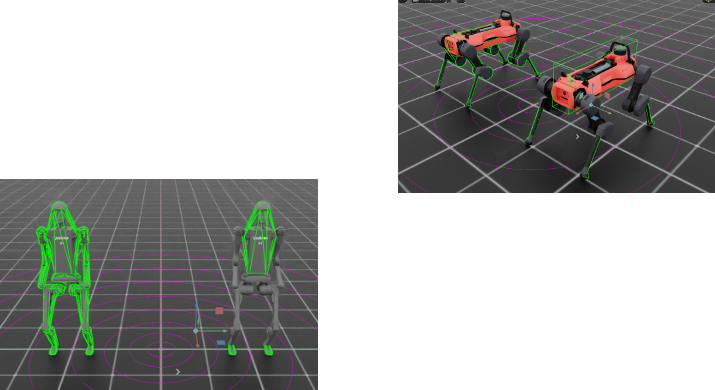}
        \end{minipage}\hfill
        \begin{minipage}[t]{0.47\linewidth}
            \centering
            {\small\textbf{(b) ANYmal-D}}\par\smallskip
            {\small\makebox[0.48\linewidth][c]{Full (\acs{hf})}\hfill\makebox[0.48\linewidth][c]{Simplified (\acs{lf})}}\par\smallskip
            \includegraphics[width=\linewidth,trim=398bp 197bp 0bp 0bp,clip]{Images/fidelity_setups/locomotion_collision_models.png}
        \end{minipage}
    \end{minipage}
    \caption{\textbf{Locomotion collision models.} \textbf{(a)} Unitree H1 and \textbf{(b)} ANYmal-D. Each panel compares the full \ac{hf} collision model (left) with the simplified \ac{lf} collision model (right). Green outlines indicate the collision geometry.}
    \label{fig:locomotion-collision-models}
\end{figure}

\begin{table}[!htb]
    \centering
    \caption{Locomotion fidelity settings, shared by the flat and rough tasks of each robot.}
    \label{tab:locomotion-fidelity-details}
    \small
    \setlength{\tabcolsep}{5pt}
    \begin{tabular}{@{}lcccc@{}}
        \toprule
        & \multicolumn{2}{c}{H1} & \multicolumn{2}{c}{ANYmal-D} \\
        \cmidrule(lr){2-3}\cmidrule(l){4-5}
        Parameter & \ac{hf} & \ac{lf} & \ac{hf} & \ac{lf} \\
        \midrule
        Physics time step (s) & $0.002$ & $0.01$ & $0.002$ & $0.0025$ \\
        Physics steps per action & 10 & 2 & 10 & 8 \\
        Position solver iterations & 4 & 2 & 4 & 2 \\
        Velocity solver iterations & 4 & 0 & 0 & 0 \\
        Self-collisions & Enabled & Disabled & Enabled & Disabled \\
        Collision geometry & Full & Simplified & Full & Simplified \\
        \bottomrule
    \end{tabular}
\end{table}

\textbf{Initialization and training length.}
For each task, all on-policy methods that use \ac{hf} data initialize from the same policy checkpoint, pre-trained with \ac{ppo} in the \ac{lf} environment for 500 iterations on H1 Rough and 100 on the other tasks.
\ac{ppo} pretraining uses Isaac Lab's task-default RSL-RL hyperparameters.
Since \ac{sac} uses different actor and critic network architectures, cf. \cref{tab:simulation-hyperparameters}, \ac{sac} \ac{hf}-Only instead initializes from its own task-specific checkpoint, pre-trained with \ac{sac} in $4096$ \ac{lf} environments using the released configurations of \citet{sabatini2026bridging}.
Among checkpoints saved every 25 iterations, we select the earliest whose \ac{lf} training return reaches that of the corresponding \ac{ppo} \ac{lf}-pretrained counterpart.
The selected checkpoints are at iterations 150, 125, 200, and 225 on H1 Rough, H1 Flat, ANYmal-D Rough, and ANYmal-D Flat, respectively.
We train for 8000 iterations on H1 Rough and 4000 on the other tasks, excluding pretraining.
\ac{mfpg}-REINFORCE uses $100$-step rollouts and trains for 1920 and 960 iterations, respectively, matching the total \ac{hf} transition counts of \ac{ppo} at the same $N^\high$.

\textbf{Simulation evaluation.}
We use Isaac Lab's \texttt{Play} evaluation tasks\footnote{The task IDs are \begingroup\def\UrlBreaks{\do\-}\path{Isaac-Velocity-Flat-H1-Play-v0}, \path{Isaac-Velocity-Rough-H1-Play-v0}, \path{Isaac-Velocity-Flat-Anymal-D-Play-v0}, and \path{Isaac-Velocity-Rough-Anymal-D-Play-v0}\endgroup.}, replacing their physics settings with the \ac{hf} configuration in \cref{tab:locomotion-fidelity-details}.
We evaluate each checkpoint deterministically over 2048 completed episodes and compute its mean undiscounted episode return.
We report means across seeds with $95\%$ bootstrapped confidence intervals.
For return differences from \ac{ppo} \ac{hf}-Only, we resample the two seed sets independently.

\textbf{Franka task.}
\Cref{fig:franka-fidelity-setup} shows the physical Franka robot \ac{hf} environment and its ManiSkill \ac{lf} simulation counterpart.
In the \ac{hf} environment, an Intel RealSense D435 camera estimates the cube's position and orientation from AprilTags attached to its six faces.
The task is to grasp a cube with edge length $3.8$\,cm and lift its center to the fixed goal $(x,y,z)=(0.45,0,0.2715)$\,m in the robot base frame.
The tabletop is at $z=0.0545$\,m, placing the goal $0.198$\,m above the cube center when the cube is on the table.
Initial cube positions lie within $x\in[0.35,0.50]$\,m and $y\in[-0.075,0.075]$\,m.
Simulation samples these coordinates uniformly, with yaw sampled uniformly from $[-\pi/4,\pi/4]$.
The physical cube may rest on any face and at any yaw. At reset, we relabel its axes using cube symmetry to express its initial orientation in the same convention.
Between physical episodes, a scripted routine repositions the cube within this region if it remains grasped. Otherwise, a human operator randomly repositions it.
Each episode comprises $H=60$ actions at $5$\,Hz and continues to the horizon after success.

The task is harder in the \ac{hf} environment because the side-mounted camera can lose sight of the cube during contact. Across both methods, approximately $69\%$ of training control steps lack a valid cube pose estimate. When perception fails, we retain the last valid pose, updating its position from the end effector if a grasp is detected. Grasp detection relies on gripper width and may be delayed. By contrast, the \ac{lf} simulator provides ground-truth object and grasp states. Friction and controller stiffness also differ between fidelities.

\begin{figure}[!htb]
    \centering
    \begin{minipage}{0.8\linewidth}
        \centering
        {\small\makebox[0.5\linewidth][c]{\textbf{Physical Franka (\acs{hf})}}\makebox[0.5\linewidth][c]{\textbf{ManiSkill (\acs{lf})}}}\par\smallskip
        \includegraphics[width=\linewidth]{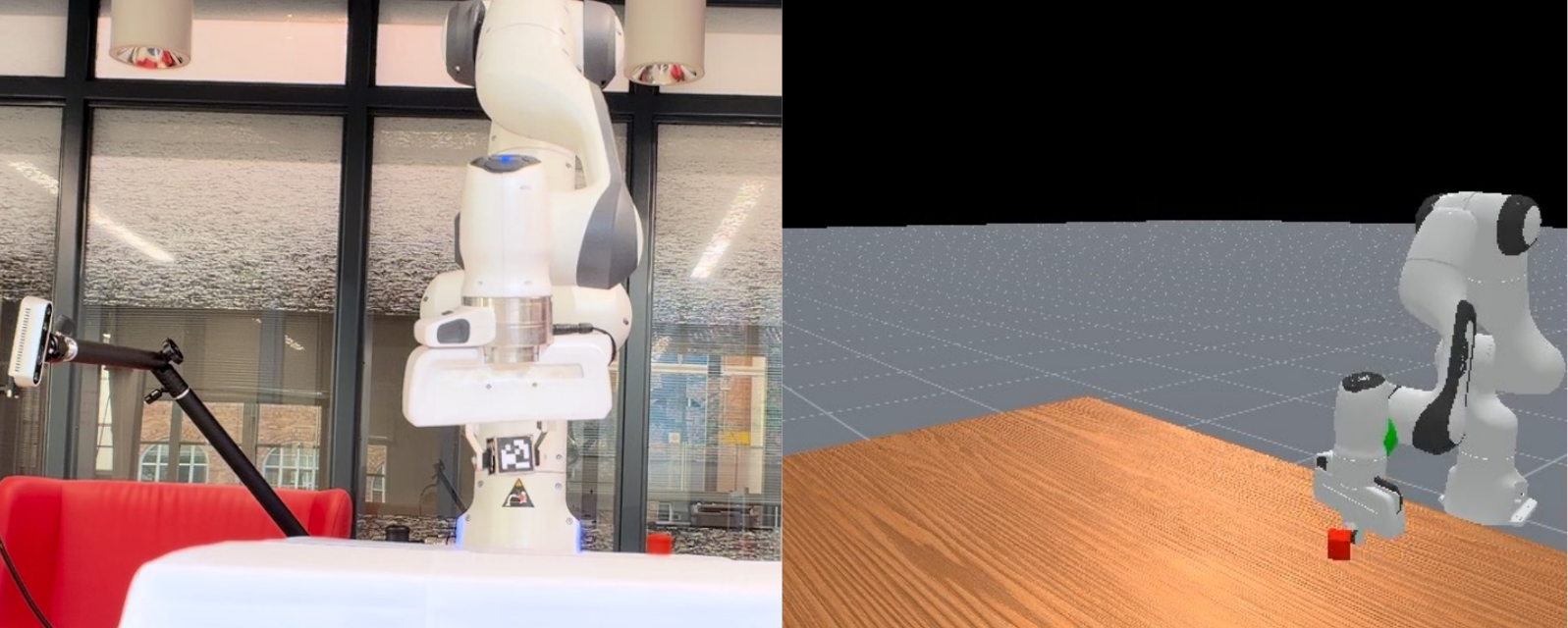}
    \end{minipage}
    \caption{\textbf{Franka cube-picking environments.} The physical Franka setup (left) provides the \ac{hf} environment, and ManiSkill simulation (right) provides the \ac{lf} environment.}
    \label{fig:franka-fidelity-setup}
\end{figure}

\textbf{Robot observations and actions.}
The policy receives the 42-dimensional state observation in \cref{tab:franka-state-observations}.
Positions are expressed in the robot base frame.
The five action components command end-effector translations, yaw rotation, and gripper opening or closing.
The physical Franka arm uses Cartesian impedance control.
The \ac{lf} simulator uses ManiSkill's default physics settings and its \texttt{pd\_ee\_delta\_pose} controller, which converts end-effector pose targets to joint targets through inverse kinematics.
Joint PD control tracks these targets, with arm-joint stiffness and damping set to $300$ and $200$, respectively.

\begin{table}[!htb]
    \centering
    \caption{State observations for the Franka cube-picking task. Here, $\varphi$ is cube yaw and $t$ counts elapsed control steps.}
    \label{tab:franka-state-observations}
    \small
    \begin{tabular}{@{}lc@{}}
        \toprule
        Observation & Dimensions \\
        \midrule
        Arm joint positions and velocities & $7+7$ \\
        Gripper width & 1 \\
        End-effector position and quaternion & $3+4$ \\
        Cube position and quaternion & $3+4$ \\
        Goal position & 3 \\
        End-effector-to-cube and cube-to-goal vectors & $3+3$ \\
        Grasp indicator & 1 \\
        Cube-yaw features $\sin(4\varphi),\cos(4\varphi)$ & 2 \\
        Episode phase $t/H$ & 1 \\
        \bottomrule
    \end{tabular}
\end{table}

\begin{samepage}
\textbf{Robot reward and success.}
Let $d_{\mathrm{reach}}$ and $d_{\mathrm{goal}}$ be the distances from the cube center to the end effector and goal, respectively, $b_{\mathrm{grasp}}\in\{0,1\}$ the grasp indicator, and $b_{\mathrm{place}}=\mathbf{1}\{d_{\mathrm{goal}}\le0.05\,\mathrm{m}\}$.
With arm joint velocities $\dot{\mathbf q}$, both fidelities use
\begin{equation*}
    \begin{aligned}
        r={}&1-\tanh(5d_{\mathrm{reach}})
        +b_{\mathrm{grasp}}\big[2-\tanh(5d_{\mathrm{goal}})\big]\\
        &+b_{\mathrm{place}}\big[1-\tanh(5\|\dot{\mathbf q}\|_2)\big].
    \end{aligned}
\end{equation*}
Success requires $d_{\mathrm{goal}}\le0.05$\,m and $\|\dot{\mathbf q}\|_\infty\le0.5$\,rad/s, and replaces the reward with 5.
An episode succeeds if this condition holds at least once.
\end{samepage}

\textbf{Other training and evaluation details.}
Both methods start from the same simulation-trained policy and collect 4 physical episodes per update.
\ac{mfpg}-\ac{ppo} additionally collects $T=20$-step uncoupled \ac{lf} simulation rollouts from $8192$ independent parallel environments as well as coupled \ac{lf} rollouts. The uncoupled \ac{lf} environments are staggered across episode phases~\citep{bharthulwar2026staggered}. We train both methods for 80 iterations per seed, using 3 seeds per method. Each trained policy is evaluated without exploration noise for 20 physical episodes.
We report means and ranges across seeds. The simulation-only reference checkpoints are each evaluated for 40 physical episodes.

\subsection{Extended Experimental Results}\label{app:more-results}

\Cref{fig:baseline-absolute-perform} shows the absolute mean \ac{hf} evaluation returns for the same settings as \cref{fig:baseline-delta}, which reports return differences relative to \ac{ppo} \ac{hf}-Only. White bars denote \ac{ppo} \ac{hf}-Only.

\begin{figure}[htbp]
    \centering
    \includegraphics[width=0.96\linewidth]{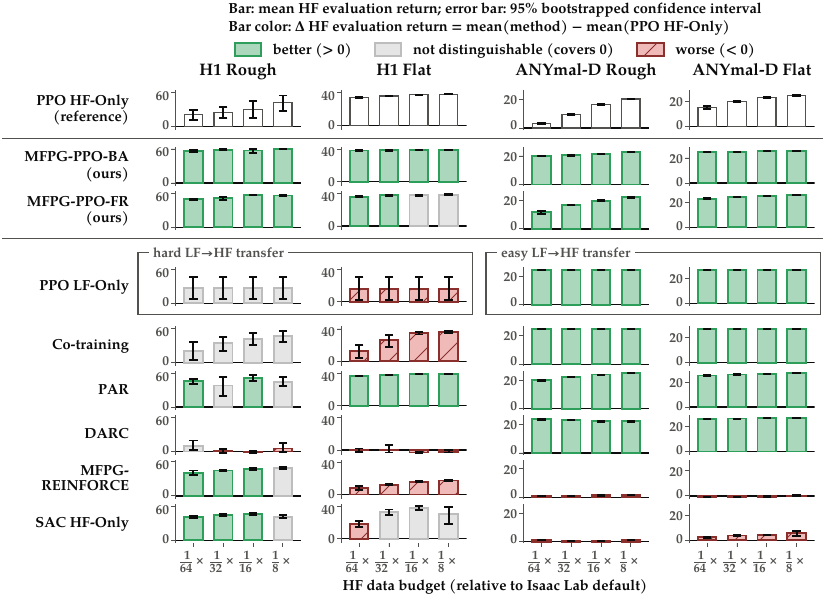}
    \caption{Mean final \ac{hf} evaluation returns for the task-budget settings in \cref{fig:baseline-delta}. Error bars denote $95\%$ bootstrap confidence intervals across 6 seeds.}
    \label{fig:baseline-absolute-perform}
\end{figure}

In \cref{fig:scaling-v3}, both \ac{mfpg}-\ac{ppo} and \ac{ppo} \ac{hf}-Only start from a common policy pretrained briefly for $500$ iterations in the \ac{lf} environment on H1 Rough, cf. \cref{app:experimental-protocol}. \Cref{fig:fine-tine-comparison} repeats this comparison at the 5 smallest \ac{hf} budgets after more extensive \ac{lf} pretraining. For each seed, both methods start from that seed's final \ac{ppo} \ac{lf}-Only checkpoint, trained for $8000$ iterations, and fine-tune for $8000$ additional iterations at matched \ac{hf} budgets using the hyperparameters in \cref{tab:simulation-hyperparameters}. \ac{mfpg}-\ac{ppo} additionally uses \ac{lf} data during fine-tuning. The gray \ac{ppo} \ac{lf}-Only bar shows the starting checkpoints' mean \ac{hf} evaluation return. \ac{mfpg}-\ac{ppo} achieves higher mean return than \ac{ppo} \ac{hf}-Only at all 5 budgets, with performance increasing before plateauing. In contrast, \ac{ppo} \ac{hf}-Only varies substantially across budgets and collapses to near-zero mean return at the 2 smallest budgets.

\begin{figure}[htbp]
    \centering
    \includegraphics[width=0.9\linewidth]{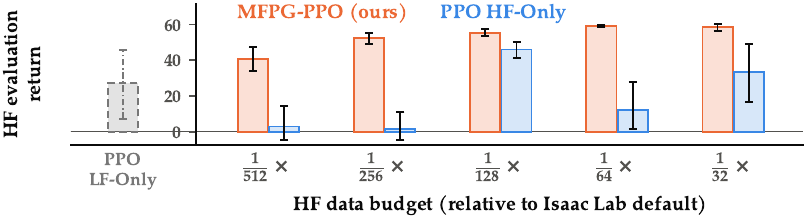}
    \caption{
Fine-tuning after extensive \ac{lf} pretraining on H1 Rough. Both methods start from each seed's final \ac{ppo} \ac{lf}-Only checkpoint, whose mean return the gray bar shows. \ac{hf} budgets are expressed as fractions of 4096 parallel environments. Bars show \ac{hf} evaluation returns over 2048 episodes per seed, averaged across 6 seeds. Whiskers show $95\%$ seed-bootstrapped confidence intervals for the means.
    }
    \label{fig:fine-tine-comparison}
\end{figure}

\Cref{fig:franka_applied_cv_percentage} shows how often MFPG-PPO-FR applies the control variate during the Franka training runs in \cref{fig:franka-main}, following the monitoring checks in \cref{app:gating}.

\begin{figure}[htbp]
    \centering
    \includegraphics[width=0.5\linewidth]{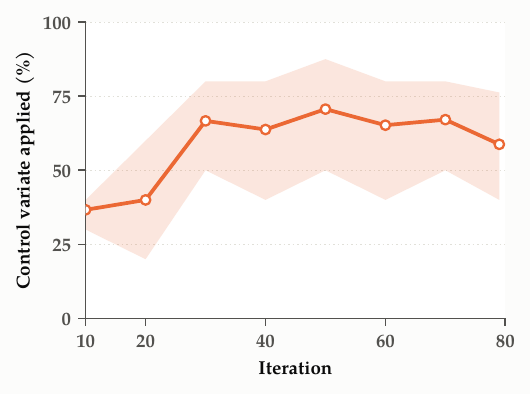}
    \caption{Percentage of minibatches on which the control variate was applied during Franka cube-picking training in \cref{fig:franka-main}. Each marker aggregates the past 10 iterations. The line is the mean over 3 seeds, the band their range.}
    \label{fig:franka_applied_cv_percentage}
\end{figure}

In Isaac Lab simulation of the locomotion tasks, the timing of key events such as robot foot strikes generally aligns more closely across fidelities, whereas grasp events on the physical Franka robot and in simulation often occur at different steps. On the physical robot, per-step correlation between paired statistics is generally low. \Cref{fig:franka-main}(c) shows how aggregating statistics over an episode captures cross-step covariance and increases the correlation, supporting the analysis in \cref{lem:fit-granularity-vi}. \Cref{fig:event_timing} reports further event-timing results and cross-fidelity correlations between paired \ac{hf} and \ac{lf} statistics.

\begin{figure}[htbp]
    \centering
    \includegraphics[width=0.9\linewidth]{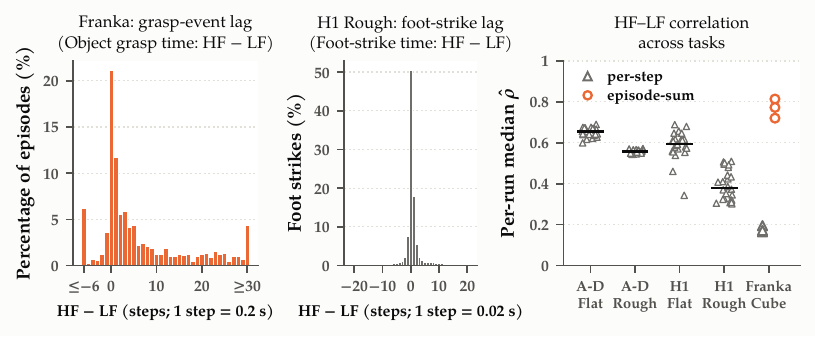}
    \caption{
    Cross-fidelity event timing and correlation between paired statistics. \textbf{Left.} Distribution of grasp-event timing differences between the physical Franka robot and its coupled ManiSkill simulation. The physical robot generally grasps later than its simulated counterpart due to fidelity gaps, causing low correlation between paired per-step statistics.
    \textbf{Middle.} Distribution of foot-strike timing differences between high- and low-fidelity simulations of the H1 Rough task in Isaac Lab. The timing generally aligns more closely than in the Franka case.
    \textbf{Right.} Cross-fidelity correlations between paired \ac{hf} and \ac{lf} statistics across the tasks in \cref{sec:baseline-comparison,sec:franka}. The first four x-axis entries are Isaac Lab simulation tasks and the last is the physical robot task.
    Each marker gives a run's median correlation over training. Gray triangles denote per-step correlations, orange circles denote episode-sum correlations, and black bars denote medians across runs. A-D denotes ANYmal-D.
    }
    \label{fig:event_timing}
\end{figure}

\Cref{fig:event_lag_illustration} illustrates how event timing appears in the paired per-step statistics $(\hfc{\RV^\high_{k,t}}, \lfc{\RV^\low_{k,t}})$ for a physical Franka cube-picking episode and an H1 Rough simulation rollout.

\begin{figure}[htbp]
    \centering
    \includegraphics[width=0.9\linewidth]{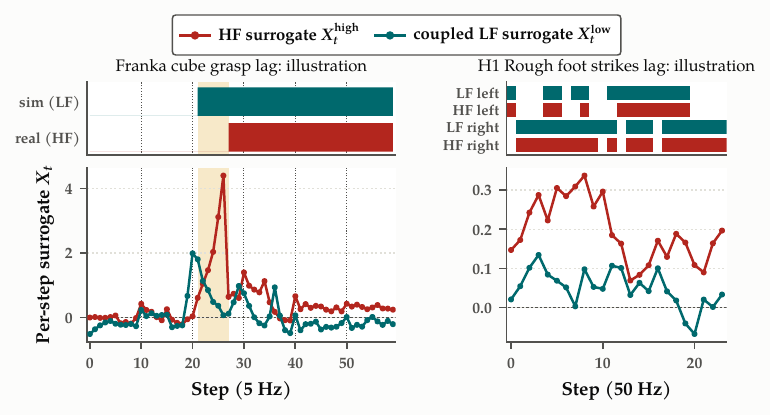}
    \caption{Event indicators (top) and paired per-step surrogate statistics $(\hfc{\RV^\high_{k,t}}, \lfc{\RV^\low_{k,t}})$ (bottom). \textbf{Left.} A Franka cube-picking episode. The upper bars indicate detected grasps, and the shaded interval marks the timing mismatch. \textbf{Right.} An H1 Rough rollout.
    The upper bars indicate left- and right-foot contact. In both panels, the lower curves show the \ac{hf} and coupled \ac{lf} surrogate statistics in red and teal, respectively.}
    \label{fig:event_lag_illustration}
\end{figure}

\Cref{fig:ablation-mechanism128} extends the variance amplification diagnostic of \cref{fig:ablation}(b) to the $\frac{1}{32}\times$ \ac{hf} budget on H1 Rough. As in \cref{fig:ablation}(b), variants without monitoring amplify variance in most iterations. 

\begin{figure}[htbp]
    \centering
    \includegraphics[width=0.5\linewidth]{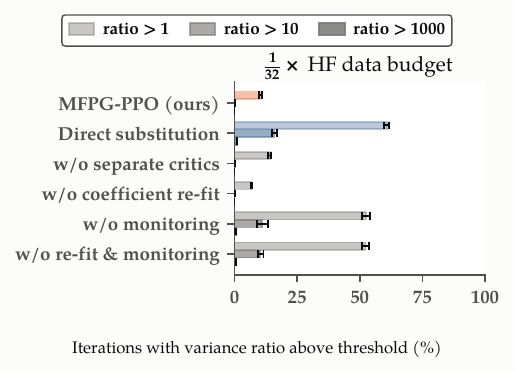}
    \caption{Fraction of training iterations whose variance ratio $\Var[\hfc{\RV^\high_{k,t}}
    + c\,(\lfc{\RV^\low_{k,t}}-\lfc{\hat{\mu}^\low})]
    \,/\,\Var[\hfc{\RV^\high_{k,t}}]$, averaged over minibatches \emph{before} the monitoring check, exceeds $1$, $10$, and $1000$ for the $\frac{1}{32}\times$ \ac{hf} budget case of \cref{fig:ablation}.}
    \label{fig:ablation-mechanism128}
\end{figure}

For the ablations in \cref{sec:ablation}, \cref{fig:ablation-diagnostics} tracks the variance ratio~$\Var[\hfc{\RV^\high_{k,t}} + c\,(\lfc{\RV^\low_{k,t}}-\lfc{\hat{\mu}^\low})]\,/\,\Var[\hfc{\RV^\high_{k,t}}]$, the fraction of minibatches failing the monitoring checks, and the estimated correlation between the paired per-step statistics $(\hfc{\RV^\high_{k,t}}, \lfc{\RV^\low_{k,t}})$. The ablation variants exhibit the same level of cross-fidelity correlation. For display, the curves for minibatch rejection and correlation are smoothed with a centered median over an 800-iteration window.

\begin{figure}[htbp]
    \centering
    \includegraphics[width=0.8\linewidth]{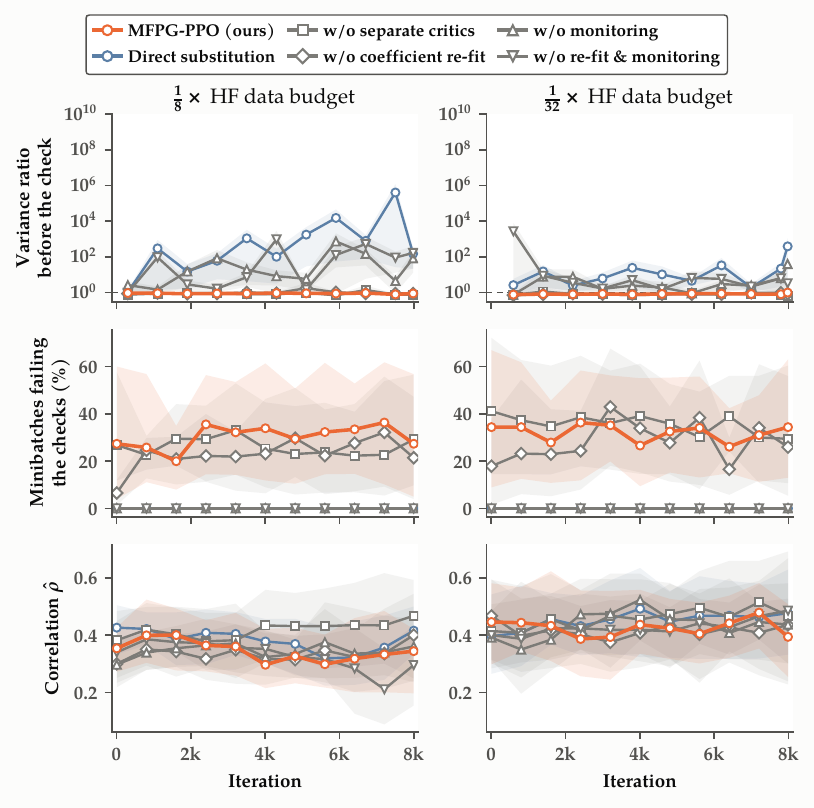}
    \caption{Training diagnostics for the ablations in \cref{sec:ablation} at the $\frac{1}{8}\times$ (left) and $\frac{1}{32}\times$ (right) \ac{hf} budgets on H1 Rough. Rows show the variance ratio~$\Var[\hfc{\RV^\high_{k,t}} + c\,(\lfc{\RV^\low_{k,t}}-\lfc{\hat{\mu}^\low})]\,/\,\Var[\hfc{\RV^\high_{k,t}}]$, the fraction of minibatches failing the monitoring checks, and the estimated correlation between the paired per-step statistics $(\hfc{\RV^\high_{k,t}}, \lfc{\RV^\low_{k,t}})$. Curves show seed means and bands show $95\%$ bootstrapped confidence intervals for the means.}
    \label{fig:ablation-diagnostics}
\end{figure}

\end{document}